%% file: MWRL_improved.tex
\documentclass[letterpaper]{article} 
\usepackage{aaai2027}  
\usepackage[hyphens]{url}
\usepackage{graphicx}
\usepackage{natbib}
\usepackage{caption}
\usepackage{algorithm}
\usepackage{algorithmic}
\usepackage{amsmath}
\usepackage{amssymb}
\usepackage{amsfonts}
\usepackage{amsthm}
\usepackage{bm}
\usepackage{booktabs}
\usepackage{tabularx}
\usepackage{float}
\usepackage{placeins}

\newtheorem{theorem}{Theorem}
\newtheorem{proposition}{Proposition}
\newtheorem{lemma}{Lemma}

\theoremstyle{definition}

\newtheorem*{collapserestate}{Theorem~\ref{thm:collapse} (Relational necessity)}
\newtheorem*{represtate}{Theorem~\ref{thm:rep} (Representation)}
\newtheorem*{supportrestate}{Proposition~\ref{prop:support} (Credit support)}
\newtheorem*{rankrestate}{Proposition~\ref{prop:rankuniform} (Rank-uniform refinement)}
\newtheorem*{groupminrestate}{Proposition~\ref{prop:groupmin} (Group minimality becomes minimality)}

\nocopyright

\newcommand{\upset}[1]{\uparrow\! #1}
\newcommand{\Msf}{\mathcal{M}}

\DeclareMathOperator*{\minop}{min}

\title{Minimal Witness Reinforcement Learning}
\author{
    T.~Y.~Tsui\textsuperscript{\rm 1},
    Zihao Ye\textsuperscript{\rm 2},
    Pengxiang Cai\textsuperscript{\rm 2},
    Yanchao Li\textsuperscript{\rm 2},
    Yuqiang Li\textsuperscript{\rm 2,\textdagger},
    Zhehong Ai\textsuperscript{\rm 2,\textdagger}
}
\affiliations{
    \textsuperscript{\rm 1}University of Pennsylvania\\
    \textsuperscript{\rm 2}Shanghai AI Laboratory\\
    \textsuperscript{\textdagger}Corresponding authors.\\
    \texttt{tuenyue.tsui@gmail.com}
}

\begin{document}
\maketitle

\begin{abstract}
``What are the irreducible conditions that are sufficient to produce an outcome?'' is one of the most common questions that recur across computation and science. Its answers, the minimal sufficient witnesses, are what we mean by explanations, mechanisms and reasons. These problems usually ask for multiple minimal witnesses, yet standard RL methods may reveal only one solution or redundant ones. We formalize this problem as minimal-witness identification and introduce Minimal-Witness
Reinforcement Learning (MWRL). MWRL takes the union of the sets certified by successful proposals sampled from the policy and credits each proposal for the coverage the group union would lose without that proposal. This credit assignment, derived directly from the problem definition, unifies the demands for minimality and recovery of alternatives from a single black-box verifier bit. Under this principle, we derive a value iteration planner that recovers the entire family of witnesses and a policy gradient method that can scale to large language models. Across different experimental settings, MWRL recovers most minimal witnesses, while other methods
return redundant supersets or a single witness.
By making witness families learnable from verifier feedback, MWRL expands the scope of
reinforcement learning beyond single-solution optimization.
Our code is available at https://github.com/TSUITUENYUE/MWRL.
\end{abstract}

\section{Introduction}

\begin{figure*}[t]
\centering
\includegraphics[width=\linewidth]{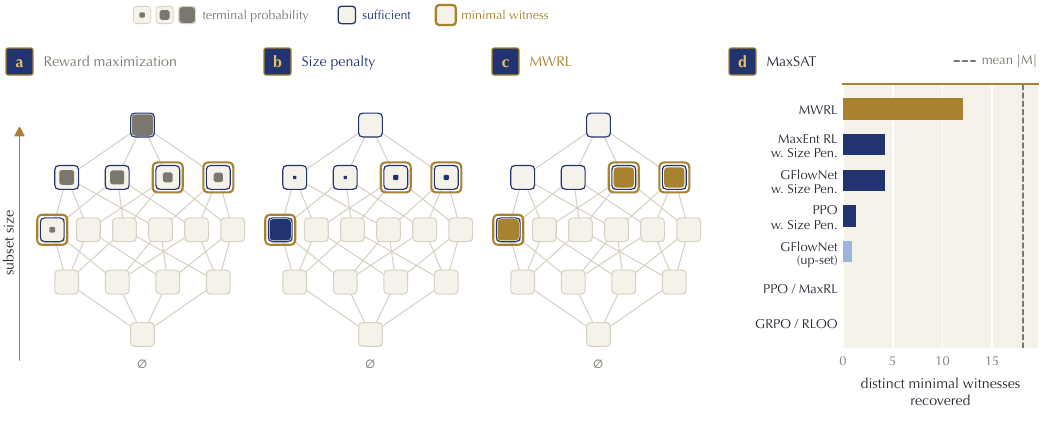}
\caption{From proposal-local rewards to the group-certified region.
\textbf{(a)--(c)} Terminal proposal probability on the subset lattice $2^{\mathcal E}$ for
$\mathcal E=\{a,b,c,d\}$, where the target antichain of minimal witnesses is ringed in gold. A success-only reward spreads
probability across all sufficient supersets (a), whereas adding a size penalty collapses the policy onto the single smallest witness
(b). MWRL credits each successful proposal for the portion of the group's certified region
that it uniquely certifies, concentrating probability directly on the full antichain (c).
\textbf{(d)} Finite-budget recovery on enumerable MaxSAT instances: MWRL recovers most of each instance's antichain, whereas baselines that evaluate each proposal in isolation recover at most a
quarter. The size-penalized MaxEnt RL and GFlowNet bars reflect the best penalty from a hyperparameter
sweep; the dashed vertical line marks the mean antichain size $|\Msf|$.}
\label{fig:failure}
\end{figure*}

Questions of sufficiency recur throughout computation and the natural sciences. In Boolean logic, we ask which conjunctions of literals force a formula to evaluate to true~\citep{quine1952}. In chemical engineering, we ask which combinations of reaction parameters push yields past a target threshold. The same structure appears in interpretability when isolating the internal components that implement a neural network behavior~\citep{conmy2023,wang2023} or the input features sufficient to sustain a model's prediction~\citep{ribeiro2018,carter2019}. In every case, the goal is to identify not merely any sufficient set of factors, but an \textbf{irreducible} one: a subset from which no element can be removed without losing sufficiency. Such an irreducible set is a \emph{minimal sufficient witness}, and most nontrivial tasks admit multiple distinct minimal witnesses.

Minimality here is easy to misunderstand as a preference for minimum size. A sufficient set is minimal when none of its proper subsets is sufficient. Different minimal witnesses are incomparable and together form an antichain of alternatives~\citep{meloux2025}. Minimality therefore defines a valid answer and cannot be replaced by a size penalty in the reward. Another misunderstanding is treating the alternatives in the antichain as a preference for diversity. Diversity methods rely on a metric, behavior descriptor or kernel~\citep{pathak2017,haarnoja2018,lehman2011,mouret2015}. In minimal-witness identification, the subset relation determines which sufficient sets belong in the answer. If the ground-truth answer contains only one witness, the best answer is to only recover that witness. If it contains three, the answer should return all three. Diversity methods, however, keep rewarding more varied solutions without deciding when the answer is complete.

Take a server diagnosis as an example. Assume it goes down when its power supply fails, when both of
its mirrored disks fail, and when one mirrored disk plus a CPU burns out. There are several concepts we can explain here. Sufficiency: the faults $\{\text{disk}_1,\text{power}\}$ together will make the server fail, but the disk failure is redundant in this set because $\{\text{power}\}$ is already sufficient. Once $\{\text{power}\}$ alone is verified, every set of faults containing it is also known to be sufficient. Minimality: the set $\{\text{power},\text{disk}_1\}$ has the same size as the pair
$\{\text{disk}_1,\text{disk}_2\}$, but only the pair is a minimal witness, while $\{\text{power}\}$ alone also counts as a minimal witness. Both $\{\text{disk}_1,\text{CPU}\}$ and $\{\text{disk}_1,\text{disk}_2\}$ are minimal witnesses because each pair causes an outage and neither component alone does. The pairs share a disk, but neither contains the other. Antichain: all irreducible
sets matter. Imagine you are the maintainer of the server. To prevent the outage from happening again, you need to know all the minimal causes, and the antichain is the set of all of them.

When the predicate is available as a formula or constraint system, we have algorithms that enumerate its minimal witnesses. These witnesses appear as prime implicants of Boolean functions~\citep{mccluskey1956}, as minimal diagnoses and unsatisfiable subsets~\citep{reiter1987,marquessilva2013msmp,liffiton2016}, and as minimal hypergraph transversals~\citep{eiter1995,fredman1996,murakami2014}. Other algorithms find them by querying a monotone predicate~\citep{bioch1995,junker2004,janota2016}. These algorithms need a deterministic predicate and run a new search for every context. However, many important settings do not expose such structure. They provide only a verifier that accepts or rejects a proposal, and we have to identify minimal witnesses from this verifier feedback alone.

Learning from binary pass/fail feedback naturally suggests Reinforcement Learning with Verifiable Rewards (RLVR), where a policy proposes candidate subsets and queries the verifier for a reward bit. Yet standard RLVR objectives fail because they \textbf{evaluate each proposal in isolation}. A pure success reward assigns identical value to a minimal witness and to all of its bloated supersets (Figure~\ref{fig:failure}a). Adding a size penalty breaks this tie in favor of smaller sets, causing the policy to collapse onto the single smallest witness while ignoring larger minimal witnesses (Figure~\ref{fig:failure}b). Even entropy-regularized RL and GFlowNets~\citep{haarnoja2018,bengio2021}, which sample proportionally to a local reward, cannot separate a true minimal witness from a non-minimal superset without starving larger minimal witnesses of probability mass. This failure is fundamental: a single verifier call on a single proposal $S$ reveals only whether $S$ is sufficient, not whether $S$ is minimal or whether it duplicates another solution.

Determining whether an accepted proposal contributes new information requires comparing it against other proposals in a rollout group. Under monotone sufficiency, verifying a subset $S$ automatically certifies all of its supersets, its \emph{up-set} in the subset lattice $2^{\mathcal E}$. If a rollout group contains two accepted proposals $S_1 \subset S_2$, the smaller proposal $S_1$ certifies a strictly larger up-set that completely contains the up-set of $S_2$, rendering $S_2$ redundant. Conversely, if two accepted proposals are incomparable ($S_1 \not\subseteq S_2$ and $S_2 \not\subseteq S_1$), each certifies a branch of supersets that the other misses.

We formalize this task as \emph{minimal-witness identification} and introduce Minimal Witness Reinforcement Learning (MWRL). Given a batch of $K$ rollouts, MWRL forms the union of the up-sets certified by all accepted proposals, the \emph{group-certified region}, and measures its volume under a full-support probability measure on $2^{\mathcal E}$. Each proposal is then assigned a leave-one-out \emph{deletion credit} equal to the certified coverage that the group would lose if that proposal were removed. This single geometric mechanism unifies both objectives of the problem:
\begin{enumerate}
\item \textbf{Pruning toward minimality:} Removing an unnecessary element from a sufficient set enlarges its up-set, strictly increasing its certified coverage. Meanwhile, any proposal that is a superset of, or identical to, another accepted proposal in the batch contributes zero marginal coverage and receives zero credit.
\item \textbf{Antichain recovery:} Discovering a minimal witness not yet represented in the batch annexes an uncovered region of the subset lattice, yielding strictly positive deletion credit (Figure~\ref{fig:failure}c).
\end{enumerate}

\paragraph{Contributions.}
\begin{itemize}
\item We formalize \emph{minimal-witness identification}: recovering the complete antichain of subset-minimal sufficient witnesses from black-box binary verifier feedback, and prove that standard proposal-separable, local reward-proportional, and success-probability RL objectives cannot solve this task in general (Section~\ref{sec:problem}).
\item We introduce Minimal Witness Reinforcement Learning (MWRL), characterizing group returns via the certified up-set union and deriving both an exact value-iteration planner and a scalable on-policy group gradient with leave-one-out deletion credit (Section~\ref{sec:method}).
\item We validate MWRL across three distinct domains: prime-implicant enumeration on monotone CNF formulas with known ground-truth antichains; amortized discovery of minimal reaction-condition sets on unseen substrates in Suzuki--Miyaura coupling; and sparse circuit discovery in a frozen Qwen3-1.7B language model under a raw activation-patching verifier that can violate monotonicity (Section~\ref{sec:experiments}).
\end{itemize}

Related work is discussed in Appendix~H, and additional problem domains fitting our formulation are detailed in Appendix~I.

\section{Problem Formulation}
\label{sec:problem}

Let $\mathcal E=\{1,\dots,n\}$ denote a finite ground set of elements (see Table~\ref{tab:notation} for a complete notation reference), and let $s:2^{\mathcal E}\to\{0,1\}$ be a \emph{sufficiency predicate}, where $s(S)=1$ if and only if the subset $S \subseteq \mathcal E$ is sufficient to produce the target outcome. We assume $s$ is monotone: adding elements to a sufficient set cannot break sufficiency, meaning $S\subseteq T$ implies $s(S)\le s(T)$. Any subset $S$ satisfying $s(S)=1$ is called a \emph{witness}. A \emph{minimal witness} is a witness for which no proper subset is also a witness. Because no minimal witness can be a subset of another, the family of all minimal witnesses forms an antichain in the subset lattice $(2^{\mathcal E}, \subseteq)$:
\begin{equation}
\label{eq:antichain}
\Msf \;=\; \minop_{\subseteq}\,\{\,S\subseteq\mathcal E : s(S)=1\,\}.
\end{equation}
In Boolean logic and combinatorics, the members of $\Msf$ correspond to the prime implicants of $s$~\citep{quine1952,crama2011}.

Given a context $c$ with sufficiency predicate $s_c$ and target antichain $\Msf(c)$, \emph{minimal-witness identification} seeks a stochastic policy $\pi_\theta(\cdot\mid c)$ whose terminal subset distribution $q_\theta(\cdot\mid c)$ recovers all members of $\Msf(c)$. We evaluate policies along two complementary axes: the \emph{born-minimal rate} (the fraction of accepted proposals that are minimal witnesses) and \emph{antichain recall} (the fraction of the ground-truth family $\Msf(c)$ recovered as distinct minimal witnesses within a fixed evaluation budget).

Because a deterministic policy can output at most one subset per context, recovering a multi-witness antichain requires a stochastic policy trained on groups of $K$ rollouts $S_{1:K}\overset{\mathrm{iid}}{\sim}q_\theta(\cdot\mid c)$ per update. For the training signal to guide $q_\theta$ toward $\Msf(c)$, the group return must be sensitive to two structural properties: (i) \emph{minimality}, whether any proper subset of an accepted proposal $S_i$ also suffices; and (ii) \emph{inter-proposal containment}, whether $S_i$ merely duplicates or contains another accepted proposal $S_j$ in the same batch.

Standard RL objectives fail to capture these relational properties. We say a group return is \emph{proposal-separable} if it decomposes as a sum of per-proposal rewards, $\sum_{i=1}^K r(S_i,s_c(S_i))$, evaluating each rollout independently of the rest of the batch. Furthermore, a reward function $r(S, s_c(S))$ is \emph{local} if it depends only on the proposed subset $S$ and its single verifier bit $s_c(S)$, remaining invariant across all monotone predicates that agree on $S$. A local reward is inherently blind to minimality: for any accepted subset $S$, the single bit $s_c(S)=1$ looks identical whether $S$ is a minimal witness or a redundant superset of a smaller witness.

\begin{theorem}[Relational necessity]
\label{thm:collapse}
Let a learning method target any of the following terminal distributions $q$ on $2^{\mathcal E}$:
\begin{enumerate}
\item the maximizers of a proposal-separable objective $J_{\mathrm{sep}}(q)=\mathbb{E}_{S_{1:K}\sim q^{\otimes K}}\big[\textstyle\sum_i r(S_i,s_c(S_i))\big]$;
\item the reward-proportional distribution $q(S)\propto r(S,s_c(S))$ for any non-negative local reward $r\ge0$; or
\item the maximizers of any objective that depends on $q$ solely through the marginal success probability $\Pr_{S\sim q}[s_c(S)=1]$.
\end{enumerate}
Then, for any ground set of size $|\mathcal E|\ge3$, there exists a monotone predicate $s_c$ with $|\Msf(c)|\ge2$ for which the target distribution's support fails to match $\Msf(c)$.
\end{theorem}

Theorem~\ref{thm:collapse} applies to three families of methods (proof in Appendix~A.3):
\begin{itemize}
\item \textbf{Proposal-separable policy gradients}, including PPO, GRPO, and RLOO~\citep{schulman2017,shao2024,ahmadian2024}, optimize an expectation $\mathbb{E}_{S\sim q}[r(S, s_c(S))]$ that is linear in $q$. Note that even though GRPO and RLOO normalize advantages across a group, their underlying expected objective remains linear in $q$ when the base reward $r$ is per-proposal. Consequently, a law that always proposes one highest-reward subset is globally optimal, and it misses all but one witness when $|\Msf(c)|\ge 2$.
\item \textbf{Proportional samplers}, including MaxEnt RL and GFlowNets~\citep{haarnoja2018,bengio2021}, match the distribution $q(S) \propto r(S, s_c(S))$. Because a local reward assigns the exact same weight $r(S,1)$ to an accepted subset $S$ regardless of whether a proper subset of $S$ is also accepted, some monotone predicate makes the sampler place positive probability on sets that are not minimal under it.
\item \textbf{Success-probability objectives} such as Maximum Likelihood RL (MaxRL)~\citep{tajwar2026maxrl} and pass@$k$ depend only on $\Pr_{S\sim q}[s_c(S)=1]$. Because monotonicity guarantees that the full ground set $\mathcal E$ is sufficient whenever any witness exists, concentrating all probability mass on $\mathcal E$ achieves the maximum possible success rate of $1$ while assigning zero probability to every nontrivial minimal witness.
\end{itemize}
Adding a size penalty to a separable objective or proportional sampler leaves the reward local, keeping the method trapped within the failure modes of Theorem~\ref{thm:collapse}. Recovering the complete antichain across all monotone predicates therefore requires a non-separable group return that explicitly compares accepted proposals within a batch.

\begin{table*}[!t]
\centering
\small
\setlength{\tabcolsep}{3pt}
\begin{tabularx}{\textwidth}{@{}l>{\raggedright\arraybackslash}X@{\hspace{1.6em}}l>{\raggedright\arraybackslash}X@{}}
\toprule
Symbol & Meaning & Symbol & Meaning \\
\midrule
\multicolumn{2}{@{}l}{\textbf{General notation}} & \multicolumn{2}{l@{}}{\textbf{Problem formulation}} \\
$\mathcal E,n$ & Ground set and its size & $s_c(S)$ & Sufficiency predicate for context $c$ \\
$S,T,Q\subseteq\mathcal E$ & Proposed or auxiliary subsets & $\Msf(c),L$ & Minimal-witness antichain and its size \\
$2^{\mathcal E},\subseteq$ & Subset lattice and inclusion order & $\minop_{\subseteq}$ & Inclusion-minimal members of a family \\
$|S|$ & Number of elements in $S$ & $F(c)$ & Set of all sufficient subsets \\
$\upset{S},\downarrow S$ & Supersets and subsets of $S$ & $s_\exists$ & Monotone closure of a raw predicate \\
$c$ & Context for the task & $r,J_{\mathrm{sep}}$ & Local reward and separable objective \\
\midrule
\multicolumn{2}{@{}l}{\textbf{Certified regions and coverage}} & \multicolumn{2}{l@{}}{\textbf{Value iteration}} \\
$W,W_K$ & Successful proposals; distinct group successes & $F,b$ & Certified region and commitments left \\
$U(W),U_K$ & Region certified by a set or group & $\Delta_{\mathcal R}(S\mid F)$ & Value of adding $\upset{S}$ beyond $F$ \\
$\mathcal R,\mathcal R_{\#}$ & Region valuation and counting measure & $V_b^\star(F)$ & Optimal value with $b$ commitments \\
$\mu,\mathcal R_\mu$ & Base measure and coverage measure $\mu(U)$ & $F_k,S_k$ & Region and choice at planner step $k$ \\
$\nu(S),p$ & Witness value $\mu(\upset{S})$ and the parameter of $\mu$ & $\lambda$ & Size penalty in scalar planning reference \\
\midrule
\multicolumn{2}{@{}l}{\textbf{Policy and sampling}} & \multicolumn{2}{l@{}}{\textbf{Group value and deletion credit}} \\
$\theta,\pi_\theta$ & Policy parameters and proposal policy & $G(W),J_K$ & Group return and expected objective \\
$x_t,a_t,\tau_i$ & State, action, and trajectory $i$ & $g,\mathcal A(W)$ & Symmetric value and minimal group successes \\
$q_\theta(S\mid c)$ & Terminal proposal law & $D_i$ & Minimal successes excluding proposal $i$ \\
$\psi(\tau),P_\theta(\tau\mid c)$ & Terminal map and trajectory law & $U_{-i},U_{-i,-j}$ & Region without indexed proposals \\
$\rho(\tau)$ & Policy-independent trajectory factors & $A_i,\widetilde A_i$ & Deletion credit and baseline-adjusted credit \\
$K,B$ & Training group size and deployment budget & $b_i$ & Baseline independent of trajectory $i$ \\
$y_i$ & Verifier bit $s_c(S_i)$ & $\widehat A_i,N$ & Monte Carlo credit and sample count \\
$\bar g_N$ & Mean gradient over $N$ rollout groups & $T^{(m)}$ & Base-measure draw for Monte Carlo credit \\
\midrule
\multicolumn{2}{@{}l}{\textbf{Derivations and diagnostics}} & \multicolumn{2}{l@{}}{\textbf{LLM circuit experiment}} \\
$\epsilon$ & Lower bound on each witness's proposal mass & $P_i,P_{S,i},P_{\varnothing,i}$ & Clean, retained, and fully ablated token laws \\
$\rho_\theta(T),p_m$ & Certification and witness probabilities & $\mathrm{KL}(S),\mathrm{KL}_{\varnothing}$ & Mean clean-to-masked and clean-to-empty KL \\
$\mathcal A_q$ & Successful subsets with positive proposal mass & $\alpha$ & Required fraction of the clean-to-empty KL gap \\
$p_t,\bar p$ & Elementwise inclusion probabilities and bound & $\mathrm{KL}_{\mathrm{held}}$ & Normalized held-out divergence retained \\
$\Phi(S)$ & Canonical witness returned by a canonicalizer & $s_0(S)$ & Raw circuit verifier \\
$\sigma$ & Standard deviation in experiment figures & $P$ & Number of probe prompts \\
\bottomrule
\end{tabularx}
\caption{Notation grouped by role in the paper. Symbols used only in appendix derivations
or diagnostics are included for completeness. We omit the context $c$ when fixed; $i,j$ index
proposals in a group, $t$ indexes policy steps, and $k$ indexes planner steps.}
\label{tab:notation}
\end{table*}

\section{Minimal Witness Reinforcement Learning}
\label{sec:method}

\subsection{Monotonicity and the certified region}
In our framework, an episodic rollout $\tau$ terminates at a proposed subset $S \subseteq \mathcal E$, which the verifier scores with a binary bit $y = s_c(S)\in\{0,1\}$. While MWRL applies to any policy parameterization whose rollouts terminate in subsets (Appendix~C.1), our experiments model the policy as a sequential walk on the subset lattice $2^{\mathcal E}$ with state $x_t=(c,S_t,t)$, either adding elements one at a time starting from $\varnothing$ or pruning elements from $\mathcal E$.

MWRL exploits the order-theoretic consequences of monotonicity. Whenever the verifier accepts a terminal subset $S$ ($s_c(S)=1$), monotonicity guarantees that every superset of $S$ is also sufficient; conversely, a rejection ($s_c(S)=0$) rules out all subsets of $S$. In lattice terminology, a single accepted proposal $S$ certifies its entire \emph{up-set}:
\begin{equation*}
\upset{S}=\{T\subseteq\mathcal E:T\supseteq S\},
\end{equation*}
and the complete family of sufficient subsets is the up-set generated by the minimal-witness antichain, $F(c)=\{S:s_c(S)=1\}=\upset{\Msf(c)}$. Given a rollout group of $K$ proposals with verifier outcomes $y_i=s_c(S_i)$, let $W=\{S_i:y_i=1\}$ denote the multiset of successful proposals. Together, the successes in $W$ certify a combined region of the subset lattice, the finitely generated up-set
\begin{equation}
U(W)=\bigcup_{S\in W}\upset{S}.
\label{eq:upset-value}
\end{equation}
This certified region $U(W)$ naturally filters out redundancy. Because subset inclusion reverses under up-sets ($S_1 \subset S_2 \implies \upset{S_1} \supset \upset{S_2}$), a smaller witness certifies a strictly larger region of the lattice than any of its supersets. If $W$ contains two accepted sets $S_1 \subset S_2$, the up-set $\upset{S_2}$ is entirely contained within $\upset{S_1}$. We call such a superset $S_2$ \emph{redundant}. Deleting redundant supersets or duplicate proposals from $W$ leaves the certified region $U(W)$ completely unchanged. The minimal elements of $U(W)$ correspond precisely to the inclusion-minimal members of $W$~\citep{davey2002}, which we call the \emph{group-minimal successes}.

\paragraph{Handling non-monotone raw verifiers.}
In some empirical domains, such as neural circuit ablation, a raw verifier $s$ may occasionally reject a superset of an accepted set due to destructive interference among components. Monotonicity can be restored formally via the \emph{existential closure}
\begin{equation*}
s_\exists(S)=\mathbf{1}[\,\exists Q\subseteq S:\ s(Q)=1\,],
\end{equation*}
which is monotone by construction and shares the exact same minimal witnesses as the raw predicate $s$. In practice, computing $s_\exists(S)$ requires searching over subsets of $S$. Instead, we can run MWRL directly on the raw verifier $s$ and evaluate a tractable local condition: an accepted set $S$ is \emph{1-minimal} if removing any single element $e \in S$ causes the verifier to reject $S \setminus \{e\}$. Under a monotone predicate, 1-minimality is equivalent to global subset-minimality. In our language-model circuit experiments, we train directly on the raw verifier and count only 1-minimal circuits, quantifying empirical monotonicity violations in Appendix~\ref{app:nonmono}.

\subsection{The group return}
Rather than scoring proposals individually, MWRL evaluates a group of rollouts collectively by measuring the volume of their certified region:
\begin{equation}
G(W)=\mathcal{R}(U(W)),
\label{eq:group-return}
\end{equation}
where $\mathcal R$ is a monotone valuation on up-sets. We train the policy to maximize the expected group coverage $J_K(\theta)=\mathbb{E}_{S_{1:K}\sim q_\theta^{\otimes K}}[\mathcal{R}(U_K)]$, where $U_K=U(W)$ is the group-certified region of $K$ rollouts (Figure~\ref{fig:failure}c). Whereas prior work on submodular or global RL optimizes non-additive utilities supplied externally by the environment~\citep{prajapat2024,desanti2024}, MWRL constructs its non-additive group return internally from binary verifier bits and the subset-inclusion order.

\begin{theorem}[Representation]
\label{thm:rep}
Let $g$ be a permutation-invariant group value function on $(S_i,y_i)_{i=1}^K$ that vanishes when no proposal succeeds, depends only on the multiset of successful proposals $W$, and remains invariant under deleting duplicate or redundant successes. Then there exists a function $\mathcal{R}$ on up-sets with $\mathcal{R}(\varnothing)=0$ such that $g=\mathcal{R}(U(W))$. Moreover, $g$ is non-decreasing when new successes are added if and only if $\mathcal{R}$ is monotone. (Proof in Appendix~A.)
\end{theorem}

Theorem~\ref{thm:rep} establishes that any group objective that properly ignores duplicates and redundant supersets must factor through the certified region $U(W)$. This follows from the one-to-one correspondence between antichains and finitely generated up-sets: two success multisets $W$ and $W'$ share the same group-minimal elements if and only if they generate the exact same up-set $U(W)=U(W')$.

To instantiate $\mathcal R$, we fix a full-support probability measure $\mu$ on $2^{\mathcal E}$, called the \emph{base measure}, and define the coverage measure as $\mathcal{R}_\mu(U)=\mu(U)$. Evaluating $\mathcal{R}_\mu(U(W))$ requires only the proposed terminal subsets, their verifier bits, and the analytical form of $\mu$; no extra verifier queries are needed. (In settings where an inexpensive local search or canonicalizer is available to shrink any accepted subset to a minimal witness, $\mathcal R$ can alternatively simply count distinct canonicalized witnesses; see Appendix~B.)

\subsection{Minimal-witness value iteration}
Before scaling to policy gradients, we show that our coverage objective induces an exact dynamic programming planner on small ground sets. Consider a sequential planner that commits one accepted subset at a time. Its state at step $k$ is the cumulative certified region $F_k \subseteq F(c)$ generated by all previously committed subsets. Committing a new accepted subset $S$ expands the certified region from $F$ to $F \cup \upset{S}$, yielding the marginal coverage gain $\Delta_{\mathcal R}(S\mid F)=\mathcal R(F\cup\upset{S})-\mathcal R(F)$. With $b$ commitment steps remaining, the optimal Bellman value is
\begin{equation}
\begin{aligned}
V_b^\star(F)&=\max_{S:\,s_c(S)=1}\big\{\Delta_{\mathcal R}(S\mid F)+
V_{b-1}^\star(F\cup\upset{S})\big\},\\
V_0^\star(F)&=0.
\end{aligned}
\label{eq:bellman}
\end{equation}
On a finite lattice, backward value iteration solves~\eqref{eq:bellman} for the optimal $b$-witness cover~\citep{bellman1957mdp}. When our goal is to enumerate the complete antichain $\Msf(c)$ without a preset budget, the planner simply selects the greedy marginal maximizer at each step:
\begin{equation}
F_{k+1}=F_k\cup\upset{S_k},\qquad
S_k\in\arg\max_{S:\,s_c(S)=1}\Delta_{\mathcal R}(S\mid F_k),
\label{eq:vi}
\end{equation}
terminating as soon as the maximum marginal gain reaches zero.

\begin{proposition}[Antichain recovery by value iteration]
\label{prop:vi}
Let $\mathcal R=\mathcal R_\mu$ for a full-support measure $\mu$, and assume the maximization in~\eqref{eq:vi} is exact. Starting from $F_0=\varnothing$, the greedy planner~\eqref{eq:vi} selects a distinct, previously undiscovered member of $\Msf(c)$ at every step and terminates in exactly $|\Msf(c)|$ steps with $F^\star=\upset{\Msf(c)}=F(c)$. (Proof in Appendix~B.)
\end{proposition}

The intuition behind Proposition~\ref{prop:vi} highlights why up-set coverage naturally enforces minimality. Any undiscovered minimal witness $M \in \Msf(c) \setminus F_k$ is guaranteed to have a strictly positive marginal gain $\Delta_{\mathcal R_\mu}(M \mid F_k) > 0$, because the singleton lattice point $\{M\}$ itself cannot be covered by the up-set of any other minimal witness. Conversely, if a non-minimal accepted set $S$ has a positive marginal gain, it must contain at least one undiscovered minimal witness $M \subset S$. Because $\upset{M} \supset \upset{S}$ and the point $M \in \upset{M} \setminus (\upset{S} \cup F_k)$ carries strictly positive mass $\mu(\{M\}) > 0$, the minimal witness $M$ always achieves a strictly larger marginal gain than its non-minimal superset $S$. Thus, exact maximization never selects a non-minimal superset and extracts the entire antichain in $|\Msf(c)|$ steps, whereas standard value iteration on a size-penalized scalar reward simply halts after finding the single smallest witness.

\subsection{The MWRL policy gradient}
Because exact value iteration requires maximizing over all $2^n$ subsets, we scale MWRL to large spaces using an on-policy group gradient that optimizes $J_K(\theta)=\mathbb{E}[\mathcal R(U_K)]$. Given a batch of $K$ rollouts terminating at subsets $S_1,\dots,S_K$ with verifier bits $y_i=s_c(S_i)$, let $U_K=\bigcup_{j:y_j=1}\upset{S_j}$ be the region certified by the full group, and let $U_{-i}=\bigcup_{j\ne i,\,y_j=1}\upset{S_j}$ be the region certified when proposal $i$ is excluded. MWRL assigns each proposal $i$ a \emph{deletion credit} equal to its leave-one-out marginal contribution to the group's certified coverage:
\begin{equation}
A_i=\mathcal R(U_K)-\mathcal R(U_{-i}).
\label{eq:advantage}
\end{equation}
Because $\mathcal R(U_{-i})$ does not depend on rollout $\tau_i$, subtracting it from the group return $\mathcal R(U_K)$ acts as an exact, zero-bias control variate in the score-function gradient for $\tau_i$. Under the coverage measure $\mathcal R=\mathcal R_\mu$, the deletion credit equals the probability mass of all supersets certified \emph{exclusively} by proposal $i$:
\begin{equation*}
A_i=y_i\,\mu(\upset{S_i}\setminus U_{-i}).
\end{equation*}
Notice the geometric parallel with the planner: while the planner in~\eqref{eq:vi} scores a candidate $S$ by its marginal gain $\Delta_{\mathcal R}(S \mid F_k)$ against previously committed witnesses, the policy gradient in~\eqref{eq:advantage} scores each rollout $S_i$ by its marginal gain $\Delta_{\mathcal R}(S_i \mid U_{-i})$ against its peers in the current rollout batch (Figure~\ref{fig:credit}). Because overlapping up-sets ($\upset{S}\cap\upset{T}=\upset{(S\cup T)}$) are counted only once in the union, $\mathcal R_\mu(U(W))$ is a monotone submodular function of the accepted proposals~\citep{nemhauser1978}, and $A_i$ is its exact discrete derivative.

\paragraph{Choosing the base measure $\mu$.}
The base measure $\mu$ governs how strongly MWRL rewards pruning an accepted proposal to a smaller subset. Define the \emph{witness value} of a single subset $S$ as the mass of its up-set, $\nu(S):=\mu(\upset{S})=\Pr_{T\sim\mu}[T\supseteq S]$. Under the uniform measure on $2^{\mathcal E}$ (which includes each element independently with probability $1/2$), $\nu(S) = 2^{-|S|}$, so removing one element doubles a proposal's witness value. More generally, we ask that removing a redundant element yields a consistent multiplicative reward boost regardless of the subset's current size $|S|$.

\begin{proposition}[Rank-uniform refinement]
\label{prop:rankuniform}
Suppose $\nu(S)=\mu(\upset{S})$ depends on $S$ only through its cardinality $|S|$. Then the single-element refinement ratio $\nu(|S|)/\nu(|S|{+}1)$ is constant across all cardinalities $|S|$ if and only if $\nu(S)=p^{|S|}$ for some parameter $p\in(0,1)$, which is uniquely realized by the product measure $\mu_p$ that includes each element $e \in \mathcal E$ independently with probability $p$. (Proof in Appendix~D.)
\end{proposition}

Under this \emph{geometric base measure} $\mu_p$, pruning one element from an accepted subset multiplies its witness value by $1/p > 1$ at every rank of the lattice. Unless stated otherwise, we fix $p=0.7$ (see Appendix~E.10 and the sensitivity sweep in Section~\ref{sec:experiments}).

\paragraph{Policy gradient and leave-two-out baseline.}
Using the deletion credit $A_i$, the unbiased policy gradient~\citep{williams1992,sutton1999} of the group objective $J_K(\theta)=\mathbb{E}[\mathcal R(U_K)]$ is
\begin{equation*}
\nabla_\theta J_K(\theta)=\mathbb{E}\Big[\sum_{i=1}^K A_i\sum_{t}
\nabla_\theta\log\pi_\theta(a_{i,t}\mid x_{i,t})\Big],
\end{equation*}
as derived in Appendix~C and summarized in Algorithm~\ref{alg:mwrl}. Because $\mathcal R$ is monotone, the raw deletion credit $A_i$ is always non-negative; using $\widetilde A_i = A_i$ directly yields the leave-one-out (\emph{l1o}) estimator. To allow negative advantages that actively push down redundant trajectories without introducing gradient bias, we can subtract a baseline $b_i$ computed solely from the remaining $K-1$ trajectories~\citep{mnih2016,ahmadian2024}. Our leave-two-out (\emph{l2o}) estimator subtracts the average marginal contribution among the other $K-1$ proposals:
\begin{equation*}
\widetilde A_i
=A_i-\Big[\mathcal R(U_{-i})-\frac{1}{K-1}\sum_{j\ne i}\mathcal R\big(U_{-i,-j}\big)\Big],
\end{equation*}
where $U_{-i,-j}$ denotes the region certified when both proposals $i$ and $j$ are omitted (see Appendix~C and Figure~\ref{fig:bias}).

\begin{algorithm}[t]
\caption{On-policy MWRL.}
\label{alg:mwrl}
\begin{algorithmic}[1]
\REQUIRE policy $\pi_\theta$, group size $K$, valuation $\mathcal R$, verifier access $s_c$
\FOR{each update}
\STATE roll out $\tau_1,\dots,\tau_K$ under $\pi_\theta$ for context $c$, yielding terminal
subsets $S_1,\dots,S_K$
\STATE query the verifier once per rollout: $y_i=s_c(S_i)$
\STATE $U_K\leftarrow\bigcup_{j:\,y_j=1}\upset{S_j}$
\STATE $U_{-i}\leftarrow\bigcup_{j\ne i,\,y_j=1}\upset{S_j}$ for each $i \in \{1,\dots,K\}$
\STATE $A_i\leftarrow\mathcal R(U_K)-\mathcal R(U_{-i})$
\STATE $\widetilde A_i\leftarrow A_i$ (optionally centered via the unbiased l2o baseline)
\STATE update $\theta$ by ascending $\sum_{i,t}\widetilde A_i\,\nabla_\theta
\log\pi_\theta(a_{i,t}\mid x_{i,t})$
\ENDFOR
\end{algorithmic}
\end{algorithm}

\begin{figure}[t]
\centering
\includegraphics[width=\linewidth]{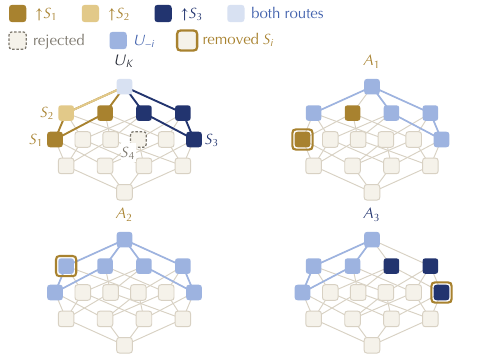}
\caption{\textbf{Deletion credit on the subset lattice}
($\mathcal E=\{a,b,c,d\}$; vertical height corresponds to subset size $|S|$). An accepted proposal certifies its up-set (all supersets above it); thus the smaller witness $S_1=\{a,b\}$ certifies a strictly larger up-set than its redundant superset $S_2=\{a,b,c\}$. In a batch where $S_1$ and $S_3$ are group-minimal successes, $S_2$ is a redundant superset, and $S_4$ fails verification, the top-left lattice shows the combined certified region $U_K$. In the bottom panels for each $A_i$, light lapis nodes mark $U_{-i}$ (the region still certified by the rest of the batch when $S_i$ is removed), and colored nodes highlight $\upset{S_i}\setminus U_{-i}$, the sets whose certification is lost without $S_i$. The base measure $\mu$ of this colored region gives the deletion credit $A_i$.}
\label{fig:credit}
\end{figure}

\begin{proposition}[Credit support]
\label{prop:support}
If $\mathcal R=\mathcal R_\mu$ for a full-support measure $\mu$, then for each rollout $i \in \{1,\dots,K\}$,
\begin{equation*}
A_i>0
\iff
y_i=1\ \text{and}\ \nexists j\ne i:\ y_j=1,\ S_j\subseteq S_i.
\end{equation*}
\end{proposition}

Proposition~\ref{prop:support} shows that deletion credit directly operationalizes the subset-inclusion order defining $\Msf(c)$ in~\eqref{eq:antichain}. Within any batch, positive gradient signal flows \emph{exclusively} to strict group-minimal successes; any duplicate proposal ($S_j = S_i$) or redundant superset ($S_j \subset S_i$) receives $A_i = 0$. Notice how this enables progressive refinement during training: early on, a non-minimal sufficient set $S_i$ may receive $A_i > 0$ if no smaller subset was sampled in the same batch, teaching the policy where the sufficient region lies. As soon as the policy discovers a proper subset $S_j \subset S_i$ within the batch, credit on $S_i$ drops to zero, and $S_j$ receives its own marginal credit. As the group size $K$ increases, group-minimality converges rapidly to global minimality:

\begin{proposition}[Group minimality becomes minimality]
\label{prop:groupmin}
Let $W_K$ denote the set of distinct accepted proposals in a batch of size $K$. If the policy assigns probability $q(m)\ge\epsilon>0$ to every minimal witness $m\in\Msf(c)$, where $L=|\Msf(c)|$, then the group-minimal successes match the ground-truth antichain, $\minop_{\subseteq}W_K=\Msf(c)$, with probability at least $1-Le^{-K\epsilon}$. (Proof in Appendix~B.)
\end{proposition}

Because a given minimal witness $m \in \Msf(c)$ is missed in $K$ independent draws with probability at most $(1-\epsilon)^K \le e^{-K\epsilon}$, a union bound over all $L$ witnesses yields Proposition~\ref{prop:groupmin}. Whenever all $L$ minimal witnesses appear in the batch, every non-minimal success contains at least one of them and receives strictly zero deletion credit (Proposition~\ref{prop:support}). At test time, when we sample the frozen policy under a deployment budget of $B$ proposals, replacing $K$ with $B$ in Proposition~\ref{prop:groupmin} bounds the probability of recovering the complete antichain (Appendix~B).

\paragraph{Fast shared-sample Monte Carlo credit.}
Computing $A_i = \mu(\upset{S_i} \setminus U_{-i})$ exactly requires accounting for overlaps among the up-sets of group-minimal successes. Exact inclusion--exclusion scales exponentially in the number of group-minimal successes, so for large batches we estimate all deletion credits and l2o baselines simultaneously using a single shared batch of $N$ auxiliary subsets drawn from $\mu_p$. Evaluating whether a sampled subset $T^{(m)} \sim \mu_p$ falls in $\upset{S_i} \setminus U_{-i}$ requires only fast bitwise subset checks in $O(NK)$ time and \textbf{zero} additional verifier calls (Appendix~D). We test the accuracy of its credits and policy gradients in the ablations.

\section{Experiments}
\label{sec:experiments}

Unless stated otherwise, we evaluate raw policies at fixed proposal budgets using
the metrics of Section~\ref{sec:problem}.

\subsection{Prime implicant enumeration}
We begin with monotone formulas in propositional logic, whose minimal witnesses can all
be enumerated. Like a fixed maze used to test an RL method, each formula gives us
a controlled instance with a known target: we can tell whether a policy found the whole family
or stopped after one valid solution, and we can compare it with the value iteration planner.
A monotone $k$-CNF instance induces the sufficiency predicate
$s(S)=[\,S\text{ satisfies every clause}\,]$, whose minimal witnesses are the prime
implicants of the monotone function, equivalently the minimal hitting sets of the
clauses~\citep{berge1989,eiter1995,eiter2003} (the MaxSAT benchmark of Appendix~E). We take $n=14$ with $16$ to $19$ minimal witnesses per
instance (mean $|\Msf|=18.1$) of sizes 2 to 4, and average over instances.
The policy adds variables one at a time and chooses when to stop.

Under a success-only reward, policies are rewarded only for their probability
of success. The policies trained with it produce redundant supersets. With a size penalty, the policy favors the smallest witness and recovers about one minimal
witness per instance.
GFlowNet~\citep{bengio2021,malkin2022} targets a law proportional to its reward and
samples many sufficient sets, but many are non-minimal supersets, leaving most of
the antichain uncovered. We add a size penalty to GFlowNet and to maximum-entropy (MaxEnt)
RL~\citep{haarnoja2018}, and both targets stay proportional to local rewards; at the best
penalty strength both recover about a quarter of the antichain, and most of their
successful proposals remain non-minimal supersets (Appendix~E.12). MWRL refines proposals along containment chains and
recovers incomparable witnesses, covering most of the antichain
(Table~\ref{tab:maxsat}). We also vary antichain size, witness size, and overlap on
planted instances (Appendix~E.9).

\begin{table}[!htbp]
\centering
\small
\setlength{\tabcolsep}{0.7pt}
\begin{tabular}{@{}lcccc@{}}
\toprule
Method & \shortstack{born-min.\\rate {\scriptsize$\uparrow$}} & \shortstack{antichain\\recall {\scriptsize$\uparrow$}} & \#found {\scriptsize$\uparrow$} & raw $|S|$ \\
\midrule
PPO / MaxRL         & $0.00{\scriptstyle\pm}0.00$ & $0.00{\scriptstyle\pm}0.00$ & $0.0{\scriptstyle\pm}0.0$  & $12.0{\scriptstyle\pm}0.0$ \\
GRPO / RLOO         & $0.00{\scriptstyle\pm}0.00$ & $0.00{\scriptstyle\pm}0.00$ & $0.0{\scriptstyle\pm}0.0$  & $12.0{\scriptstyle\pm}0.0$ \\
GFlowNet (success)  & $0.01{\scriptstyle\pm}0.00$ & $0.02{\scriptstyle\pm}0.02$ & $0.4{\scriptstyle\pm}0.3$  & $7.8{\scriptstyle\pm}0.1$ \\
GFlowNet (up-set)   & $0.01{\scriptstyle\pm}0.00$ & $0.05{\scriptstyle\pm}0.02$ & $0.9{\scriptstyle\pm}0.3$  & $7.0{\scriptstyle\pm}0.4$ \\
\midrule
\multicolumn{5}{l}{\emph{with a size penalty}}\\
PPO                 & $1.00{\scriptstyle\pm}0.00$ & $0.07{\scriptstyle\pm}0.01$ & $1.3{\scriptstyle\pm}0.1$  & $2.8{\scriptstyle\pm}0.0$ \\
MaxEnt RL           & $0.12{\scriptstyle\pm}0.01$ & $0.24{\scriptstyle\pm}0.02$ & $4.3{\scriptstyle\pm}0.3$ & $4.8{\scriptstyle\pm}0.1$ \\
GFlowNet            & $0.07{\scriptstyle\pm}0.00$ & $0.24{\scriptstyle\pm}0.01$ & $4.3{\scriptstyle\pm}0.2$ & $5.0{\scriptstyle\pm}0.1$ \\
\midrule
MWRL                & $\mathbf{0.76{\scriptstyle\pm}0.01}$ & $\mathbf{0.67{\scriptstyle\pm}0.01}$ & $\mathbf{12.1{\scriptstyle\pm}0.3}$ & $4.0{\scriptstyle\pm}0.1$ \\
\midrule
\multicolumn{5}{l}{\emph{planning references (full information)}}\\
Scalar VI           & $1.00$ & $0.06$ & $1.0$  & $2.8$ \\
Minimal-witness VI  & $1.00$ & $1.00$ & $18.1$ & $4.1$ \\
\bottomrule
\end{tabular}
\caption{Prime-implicant recovery with ground truth. Columns report born-minimal
rate, antichain recall, the number of distinct witnesses found, and raw proposal size.
Scalar VI is value iteration on a size-penalized scalar reward and returns one witness. Minimal-witness VI plans
on $(F,b)$ and recovers the full antichain at $b=|\Msf|$
(Appendix~E.4). Learned rows report mean $\pm$ standard deviation over
three policy seeds after averaging the eight instances (training configuration in
Appendix~E.3). The size-penalized MaxEnt RL and GFlowNet rows report the
best setting of a sweep over the penalty strength, trained with $57{,}600$ verifier calls
per formula.}
\label{tab:maxsat}
\end{table}

\subsection{Scientific variable identification}
\begin{figure}[t!]
\centering
\includegraphics[width=\linewidth]{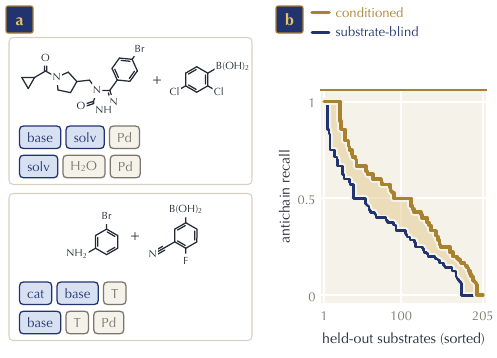}
\caption{Amortized discovery of minimal reaction-condition sets on the Suzuki
coupling. \textbf{(a)} Two held-out substrate pairs with their ground-truth antichains.
Each row of chips names one minimal set of dimensions allowed to vary from the
baseline (cat catalyst system, solv solvent, Pd loading, T temperature). \textbf{(b)} Recall on each held-out substrate
within the window of 32 distinct condition sets, with the values of each policy sorted in decreasing order. The
shaded gap between the conditioned and the substrate-blind policy is the recall that the
substrate fingerprint adds.}
\label{fig:amortization}
\end{figure}

We train one policy across many Suzuki--Miyaura substrate pairs and test it on
pairs withheld during training. Each pair is a context, and the policy receives a
fingerprint of the substrate pair as input.

The variables are 14 reaction-condition dimensions defined relative to a shared
baseline protocol. For a substrate pair $c$, a subset $S$ is sufficient when the
benchmark contains a reaction that changes only dimensions in $S$ and reaches the
target yield. Conditions outside $S$ remain at baseline. A minimal witness is therefore an inclusion-minimal set of dimensions that can be
varied to achieve the target yield. One substrate pair can admit several
incomparable minimal witnesses.

The verifier evaluates this predicate by closed-world lookup in the
benchmark reaction table. It returns $s_c(S)=1$ if and only if a tabulated reaction
with deviation set contained in $S$ reaches the yield threshold; unlisted assignments
count as failures. MWRL receives only this verifier bit as the reward. In a wet-lab deployment, the same verifier can be
implemented by Bayesian optimization restricted to $S$, with each query executing a
reaction and measuring its yield.

We build the dataset by augmenting real literature reactions from a curated
catalog~\citep{chemrang2026} that supplies the substrate pairs, products, and condition
vocabulary. However, the catalog does not contain complete condition panels. We construct the panels using mechanistic
chemical rules calibrated on measured flow-chemistry screens~\citep{perera2018} (Appendix~F).

We show that the conditioned policy recovers minimal condition sets for substrate pairs it
never saw during training. Its held-out recall is close to its seen recall and
reaches a large fraction of the recall obtained by training a separate policy for
each substrate (Table~\ref{tab:amortization}, Figure~\ref{fig:amortization}). Its
born-minimal rate also exceeds that of these separately trained policies. Without
the substrate fingerprint, recall falls to about half their recall; a scalar reward does not recover any minimal witness.

\begin{table}[!htbp]
\centering
\small
\setlength{\tabcolsep}{2pt}
\begin{tabular}{@{}lcccc@{}}
\toprule
Method & born-min. & \multicolumn{2}{c}{recall {\scriptsize$\uparrow$}} & \#found {\scriptsize$\uparrow$} \\
\cmidrule(lr){3-4}
       & rate {\scriptsize$\uparrow$} & seen & held-out & \\
\midrule
Per-substrate & $0.20$ & $0.67$ & --- & $4.9$ \\
\midrule
\multicolumn{5}{@{}l}{\emph{Amortized MWRL}} \\
\quad Conditioned & $0.26{\scriptstyle\pm}0.02$ & $0.50{\scriptstyle\pm}0.02$ & $\mathbf{0.48{\scriptstyle\pm}0.03}$ & $\mathbf{3.6{\scriptstyle\pm}0.2}$ \\
\quad Substrate-blind & $0.20{\scriptstyle\pm}0.01$ & $0.37{\scriptstyle\pm}0.01$ & $0.34{\scriptstyle\pm}0.02$ & $2.6{\scriptstyle\pm}0.1$ \\
\midrule
Scalar RL & $0.00{\scriptstyle\pm}0.00$ & $0.00{\scriptstyle\pm}0.00$ & $0.00{\scriptstyle\pm}0.00$ & $0.0{\scriptstyle\pm}0.0$ \\
\bottomrule
\end{tabular}
\caption{Amortized recovery of minimal reaction-condition sets within a window of 32
distinct condition sets per substrate. We report antichain recall on substrates seen in training
and on held-out substrates, and the held-out column measures amortization. The column \#found
counts distinct minimal witnesses per held-out context. The two amortized rows train the
same MWRL objective and differ only in access to the substrate fingerprint. Amortized rows
give the mean $\pm$ standard deviation over three seeds that control both the substrate split
and the policy. Per-substrate search retrains for each substrate to mark the ceiling, and
nothing is held out from it.}
\label{tab:amortization}
\end{table}

\subsection{LLM sparse circuit discovery}
\label{sec:circuits}
We search subsets of attention heads and MLP blocks in a frozen Qwen3~\citep{qwen3}. To
test a subset, we give every component outside it the activation it has on another prompt
of the same task. We measure the KL divergence between the resulting next-token distribution and
that of the unmodified model. The subset is accepted when this divergence stays
below $1-\alpha=0.2$ times the corresponding divergence with every component
replaced~\citep{conmy2023}.
We use the raw verifier even though it can reject a superset of an accepted
circuit (Appendix~\ref{app:nonmono}).

We compare recovered circuits with circuit-discovery and pruning methods on MMLU
under matched sparsity (Appendix~G.1).

On Qwen3-1.7B ($476$ components), we evaluate the dense model on all $57$ MMLU
subjects~\citep{hendrycks2021} and retain the ten on which the correct answer receives higher probability than a foil answer on at least
$0.92$ of the probes. We compare against top-$k$ selection by EAP attribution~\citep{syed2024}, weight
magnitude, and the activation-aware Wanda score~\citep{sun2024}. We also include a
per-subject HardConcrete mask trained through differentiable
ablation~\citep{louizos2018}, greedy ACDC-style faithfulness
pruning~\citep{conmy2023}, and random-$k$. Every method uses the same
subject-specific active-component budget, except ACDC, which uses its converged size.
We discover circuits on one prompt split, then freeze them and evaluate them on
questions drawn from disjoint dataset rows.

\begin{figure*}[t]
\centering
\includegraphics{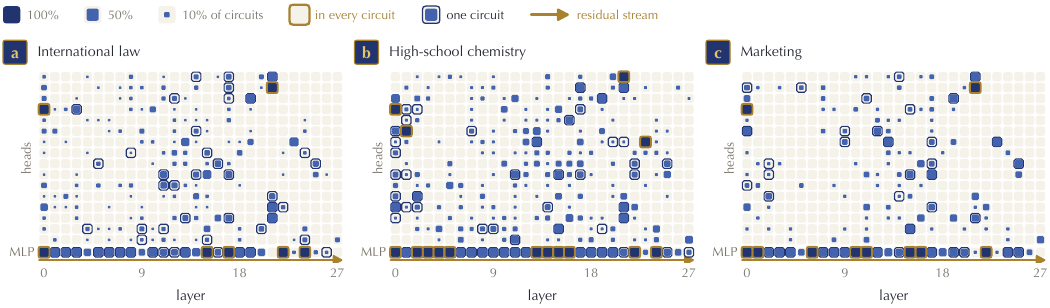}
\caption{Complete recovered circuit families for three MMLU subjects on Qwen3-1.7B,
shown on the module graph ($28$ layers $\times$ $16$ heads, MLP rail below). The recovered families contain $14$, $10$ and $8$ circuits, with median sizes $55$, $67$ and
$49$, respectively. A tessera's area is the share of the subject's recovered circuits that use the module, and the lapis outline marks one circuit of median size. Gold rims mark the modules present in every recovered circuit: $7$ of
the $189$ modules used in (a), $15$ of $230$ in (b), and $10$ of $136$ in (c). The remaining
modules vary across circuits.}
\label{fig:antichain}
\end{figure*}

On the discovery prompts, only the verifier-guided searches find sufficient circuits at
the matched budget. ACDC returns one circuit per subject, and MWRL returns a family of
circuits with similar discovery retention (Figure~\ref{fig:antichain}). On held-out
questions, we measure how much of the original behavior each circuit keeps with
KL$_{\text{held}}$ (Table~\ref{tab:circuits}). MWRL circuits keep substantially more of
the original behavior than ranking-based masks.

\begin{table}[!t]
\centering
\small
\setlength{\tabcolsep}{1.5pt}
\begin{tabular}{lccc}
\toprule
Method & KL$_{\text{held}}$ $\uparrow$ & retention $\uparrow$ & circuits $\uparrow$ \\
\midrule
Dense model        & $1.00$ & $1.00$ & --- \\
Random-$k$         & $-0.01{\scriptstyle\pm}0.01$ & $0.18{\scriptstyle\pm}0.00$ & --- \\
Magnitude          & $0.02{\scriptstyle\pm}0.00$  & $0.21{\scriptstyle\pm}0.00$ & --- \\
Wanda              & $-0.05{\scriptstyle\pm}0.00$ & $0.20{\scriptstyle\pm}0.01$ & --- \\
EAP                & $0.15{\scriptstyle\pm}0.01$  & $0.20{\scriptstyle\pm}0.02$ & --- \\
ACDC (greedy)      & $\mathbf{0.42{\scriptstyle\pm}0.01}$ & $\mathbf{0.74{\scriptstyle\pm}0.01}$ & $1$ \\
HardConcrete       & $0.02{\scriptstyle\pm}0.00$  & $0.18{\scriptstyle\pm}0.01$ & $1$ \\
MWRL               & $0.38{\scriptstyle\pm}0.02$ & $\mathbf{0.74{\scriptstyle\pm}0.04}$ & $\mathbf{8.9{\scriptstyle\pm}1.3}$ \\
\bottomrule
\end{tabular}
\caption{Circuit recovery on ten MMLU subjects with Qwen3-1.7B, every mask at the subject's matched
active-component budget (ACDC at its converged size). All columns evaluate the frozen
circuit with every other component given its activation on another prompt.
KL$_{\text{held}}$ is one minus the circuit's KL divergence from the unmodified model,
divided by the divergence when every component is replaced, on questions disjoint from
discovery ($1$ = the circuit behaves like the full model, $0$ = no better than replacing
every component); retention is discovery-prompt accuracy relative
to the dense model; circuits is the number of distinct proposals per subject that are 1-minimal under
the raw verifier. Means over
ten subjects and three discovery seeds, $\pm$ the across-seed standard deviation.}
\label{tab:circuits}
\end{table}

\subsection{Ablations}

\subsubsection{Credit estimator.}
The raw deletion credit (l1o) and the leave-two-out baseline (l2o) of Section~\ref{sec:method} both give unbiased gradients. Group-mean centering, which subtracts $\frac{1}{K}\sum_{j=1}^K A_j$, includes $A_i$ in rollout $i$'s own baseline. Figure~\ref{fig:bias} compares the three with the exact policy gradient on an enumerable MaxSAT instance: l1o and l2o converge to it at the Monte Carlo rate $O(1/\sqrt{N})$, whereas group-mean centering plateaus at a nonzero bias.

\begin{figure}[!t]
\centering
\includegraphics[width=\linewidth]{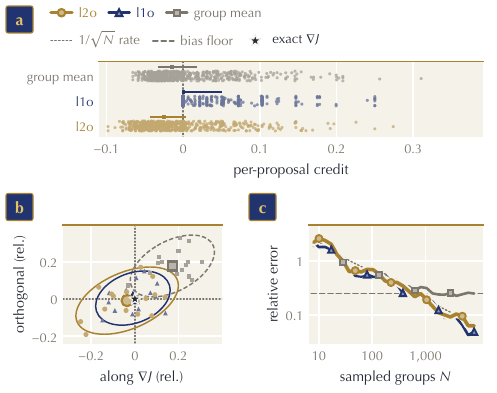}
\caption{Comparison of credit estimators against the exact analytical policy gradient on an enumerable MaxSAT
instance (frozen mid-training policy, $K=6$). \textbf{(a)} Per-proposal credit distributions across shared rollout batches: raw deletion credit (l1o) is non-negative, whereas l2o and group-mean centering produce signed advantages. \textbf{(b)} 2D projection of gradient estimates: small dots show averages over $512$ groups, large markers show the mean across $16$ independent runs with $2\sigma$ dispersion ellipses, and the star marks the exact gradient $\nabla J$.
\textbf{(c)} Relative gradient error $\lVert\bar g_N-\nabla J\rVert/\lVert\nabla J\rVert$ vs.\ number of rollout groups $N$: both l1o and l2o decay at the unbiased $1/\sqrt N$ rate (dotted line), while group-mean centering stalls at its finite-group bias floor (dashed line).}
\label{fig:bias}
\end{figure}

\subsubsection{Training group size and deployment budget.}
Larger training groups $K$ increase the likelihood that a batch simultaneously contains a minimal witness and its non-minimal supersets, allowing deletion credit to zero out the supersets immediately (Proposition~\ref{prop:groupmin}). Holding the total number of verifier queries strictly fixed at $36{,}864$ per instance, so that larger $K$ means fewer gradient steps, antichain recovery still rises with $K$ (Figure~\ref{fig:kmatched}a). At deployment time, varying the evaluation budget $B$ on frozen policies reveals a natural trade-off: smaller training groups produce sharper policies that excel at tiny deployment budgets $B$, whereas larger training groups distribute probability more evenly across the antichain and dominate at moderate-to-large $B$ (Figure~\ref{fig:kmatched}b).

\begin{figure}[!t]
\centering
\includegraphics[width=\linewidth]{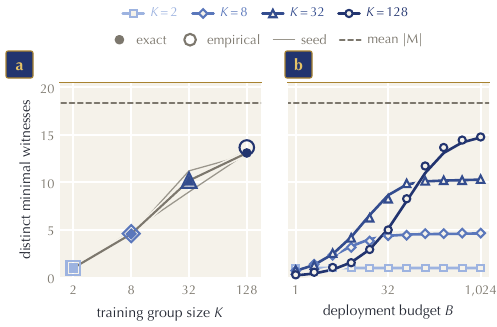}
\caption{Effect of training group size $K$ and deployment budget $B$ under a fixed training budget of $36{,}864$
verifier calls per instance (dashed horizontal line indicates true antichain size $|\Msf|$). \textbf{(a)} Witnesses recovered at deployment budget $B=256$ vs.\ training group size $K$: grey curves show three individual seeds averaged across six instances, filled markers show seed means, and open rings show empirical estimates. \textbf{(b)} Deployment budget sweep $B$ for frozen policies trained at different $K$: solid curves plot the exact expected number of distinct minimal witnesses in $B$ draws (Appendix~\ref{app:kinf-limit}), and markers show empirical recovery.}
\label{fig:kmatched}
\end{figure}

\subsubsection{Shared-sample Monte Carlo credit.}
When a rollout group contains more than $12$ group-minimal successes, MWRL switches from exact inclusion--exclusion to our shared-sample Monte Carlo estimator using $N$ auxiliary draws from $\mu_p$. Table~\ref{tab:mc-fidelity} evaluates the fidelity of this estimator for $N\in\{256,1024,5000,20000\}$ against exact lattice enumeration ($2^{14}$ states) on a challenging $K=48$ batch containing $32$ successes ($8$ with positive credit, alongside duplicates, redundant supersets, and failures). At $N=5{,}000$, which requires only bitwise subset checks and zero verifier calls, the Monte Carlo estimator achieves $1.000$ support recall and $0.995$ cosine similarity on both the l2o credit vector and the resulting policy gradient.

\begin{table}[!htbp]
\centering
\small
\setlength{\tabcolsep}{2.5pt}
\begin{tabular}{@{}ccccc@{}}
\toprule
$N$ & $\cos(\widehat{\widetilde A},\widetilde A)\uparrow$ &
\shortstack{relative $\ell_1$\\error $\downarrow$} &
\shortstack{support\\recall $\uparrow$} &
\shortstack{policy-grad.\\cosine $\uparrow$} \\
\midrule
$256$ & $0.866{\scriptstyle\pm}0.138$ & $0.516{\scriptstyle\pm}0.176$
& $0.594{\scriptstyle\pm}0.140$ & $0.888{\scriptstyle\pm}0.163$ \\
$1{,}024$ & $0.970{\scriptstyle\pm}0.027$ & $0.232{\scriptstyle\pm}0.069$
& $0.850{\scriptstyle\pm}0.150$ & $0.972{\scriptstyle\pm}0.020$ \\
$5{,}000$ & $0.995{\scriptstyle\pm}0.004$ & $0.094{\scriptstyle\pm}0.033$
& $1.000{\scriptstyle\pm}0.000$ & $0.995{\scriptstyle\pm}0.004$ \\
$20{,}000$ & $0.998{\scriptstyle\pm}0.001$ & $0.059{\scriptstyle\pm}0.017$
& $1.000{\scriptstyle\pm}0.000$ & $0.999{\scriptstyle\pm}0.001$ \\
\bottomrule
\end{tabular}
\caption{Fidelity of shared-sample Monte Carlo credit estimation against exact lattice enumeration under the geometric base measure ($p=0.7$). Credit metrics are averaged over $20$ auxiliary sample batches for a fixed $K=48$ group; policy-gradient cosine similarity is averaged across $12$ on-policy groups. The first two columns evaluate baseline-adjusted credits $\widetilde A$, support recall evaluates raw positive credits $A_i > 0$, and $N$ counts auxiliary draws from $\mu_p$ (requiring zero verifier calls).}
\label{tab:mc-fidelity}
\end{table}

\subsubsection{Base-measure choice.}
In the geometric base measure $\mu_p$, each removed element multiplies a witness's up-set value $\nu(S) = p^{|S|}$ by $1/p$ (Proposition~\ref{prop:rankuniform}). Smaller values of $p$ create a steeper preference for small witnesses, whereas larger $p$ equalizes witness values across different cardinalities. Sweeping $p$ on the MaxSAT benchmark (Figure~\ref{fig:psweep-main}) confirms that the born-minimal rate remains stably high across the entire range, while antichain recall increases with $p$ as larger minimal witnesses receive more balanced gradient weight relative to smaller ones. We fixed $p=0.7$ before this sweep, near the value $1/\sqrt2$, which has equal logarithmic distances to strong refinement ($p=\tfrac12$) and to no refinement preference ($p=1$) (Appendix~E.10).

\begin{figure}[!htbp]
\centering
\includegraphics{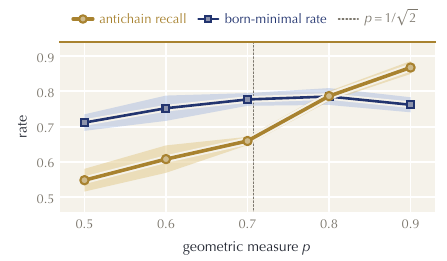}
\caption{Sensitivity sweep over the geometric base-measure parameter $p$ on the MaxSAT benchmark (dotted vertical line marks $p=1/\sqrt2 \approx 0.707$). Curves show mean $\pm$ standard deviation across three seeds and eight instances.}
\label{fig:psweep-main}
\end{figure}

\section{Conclusion}
We formalized minimal-witness identification, the recovery of every minimal sufficient
set of a task from binary verifier feedback, and introduced MWRL to solve it. MWRL values
the region certified by a group of proposals and credits each proposal for the part of
this region that would be lost without it. From this objective, we derived a planner that
recovers every minimal witness on enumerable instances and a group policy gradient that
scales to language models. MWRL recovers most minimal witnesses of each task, while methods
that value each proposal in isolation return redundant supersets or a single witness. One
trained policy also recovers minimal witnesses on tasks unseen in training. Wherever a
verifier certifies an outcome, the family of minimal witnesses of that outcome becomes a
learnable object.

\bibliography{aaai2027}

\appendix
\raggedbottom
\makeatletter
\setlength{\@fptop}{0pt}
\setlength{\@fpsep}{10pt plus 4pt}
\setlength{\@fpbot}{0pt plus 1fil}
\setlength{\@dblfptop}{0pt}
\setlength{\@dblfpsep}{10pt plus 4pt}
\setlength{\@dblfpbot}{0pt plus 1fil}
\makeatother
\renewcommand{\topfraction}{0.92}
\renewcommand{\dbltopfraction}{0.92}
\renewcommand{\bottomfraction}{0.65}
\renewcommand{\textfraction}{0.05}
\renewcommand{\floatpagefraction}{0.85}
\renewcommand{\dblfloatpagefraction}{0.85}
\input{MWRL_Appendix_Improved}

\end{document}

%% file: MWRL_Appendix_Improved.tex

\section{Foundations and verifier scope}
\label{app:representation}

\subsection{Up-set representation}
Under a monotone predicate, an accepted proposal certifies all of its supersets and a rejected
proposal rules out all of its subsets: $s_c(S)=1\Rightarrow\upset{S}\subseteq
F(c)$ and $s_c(S)=0\Rightarrow\downarrow S\cap F(c)=\varnothing$, where $\downarrow
S=\{T:T\subseteq S\}$. Across a group of $K$ rollouts, the accepted proposals therefore certify the combined region $U_K\subseteq F(c)$.

\begin{represtate}
Let $g$ be a permutation-invariant group value function on $(S_i,y_i)_{i=1}^K$ that vanishes when no proposal succeeds, depends only on the multiset of successful proposals $W$, and remains invariant under deleting duplicate or redundant successes. Then there exists a function $\mathcal{R}$ on up-sets with $\mathcal{R}(\varnothing)=0$ such that $g=\mathcal{R}(U(W))$. Moreover, $g$ is non-decreasing when new successes are added if and only if $\mathcal{R}$ is monotone.
\end{represtate}
\begin{proof}
Because $g$ depends only on the multiset of successful proposals $W$, we can write $g=g(W)$. Deleting duplicate proposals and redundant supersets
leaves the value unchanged, so $g(W)=g(\mathcal{A}(W))$, where
$\mathcal{A}(W)=\minop_\subseteq(\operatorname{supp}W)$ denotes the inclusion-minimal members of the set of distinct successes $\operatorname{supp}W$. Antichains and finitely generated up-sets stand in one-to-one correspondence: the maps $\mathcal{A}\mapsto\upset{\mathcal{A}}$ and
$U\mapsto\minop_\subseteq U$ are mutual inverses. Consequently, defining
$\mathcal R(U):=g(\minop_\subseteq U)$ yields a well-defined set function on up-sets satisfying
$g(W)=\mathcal R(U(W))$. Finally, because adding new successes to $W$ expands $U(W)$ by set inclusion, $g$ is non-decreasing in added successes if and only if $\mathcal R$ is monotone.
\end{proof}

Thus, the certified region $U_K$ completely determines the group return, and the policy objective takes the form
$J_K=\mathbb{E}[\mathcal R(U_K)]$.

\subsection{Group returns and verifier access}
Throughout the main text, we instantiate $\mathcal R$ using the coverage measure $\mathcal R_\mu(U)=\mu(U)$, which values the certified region directly using only the single verifier bit per rollout. As an alternative when additional verifier queries are permitted, the counting measure first shrinks each successful proposal to a minimal witness via a canonicalizer $\Phi$ and then assigns unit weight to each distinct canonical witness recovered in the group (Appendix~B).

\subsection{Proposal-separable objectives}
\label{app:separable}
A symmetric group return is proposal-separable if, for each group size $K$, there exist a constant
$b_K(c)$ and a per-proposal reward function
$r_c:2^{\mathcal E}\times\{0,1\}\to\mathbb{R}$ such that
\begin{equation*}
G(S_{1:K},y_{1:K})=b_K(c)+\sum_{i=1}^K r_c(S_i,y_i).
\end{equation*}
Here each term $r_c(S_i,y_i)$ depends on proposal $S_i$ and its own verifier bit $y_i$, but not on the other proposals in the batch. Because the additive constant $b_K(c)$ shifts all distributions equally, we set $b_K(c)=0$ without loss of generality.

Beyond proposal-separable returns, a method may target a reward-proportional distribution $q\propto r(\cdot,s_c(\cdot))$ for a non-negative local reward $r\ge0$, where a reward is local if it is the same function of $(S,s_c(S))$ across all monotone predicates. Alternatively, a method may optimize an objective that depends on $q$ solely through the marginal success rate
$q(F(c))=\Pr_{S\sim q}[s_c(S)=1]$. Theorem~\ref{thm:collapse} addresses all three families simultaneously.

\begin{collapserestate}
Let a learning method target any of the following terminal distributions $q$ on $2^{\mathcal E}$:
\begin{enumerate}
\item the maximizers of a proposal-separable objective $J_{\mathrm{sep}}(q)=\mathbb{E}_{S_{1:K}\sim q^{\otimes K}}\big[\textstyle\sum_i r(S_i,s_c(S_i))\big]$;
\item the reward-proportional distribution $q(S)\propto r(S,s_c(S))$ for any non-negative local reward $r\ge0$; or
\item the maximizers of any objective that depends on $q$ solely through the marginal success probability $\Pr_{S\sim q}[s_c(S)=1]$.
\end{enumerate}
Then, for any ground set of size $|\mathcal E|\ge3$, there exists a monotone predicate $s_c$ with $|\Msf(c)|\ge2$ for which the target distribution's support fails to match $\Msf(c)$.
\end{collapserestate}
\begin{proof}
We assume each rule is well-defined across all monotone predicates: the success-probability objective attains its maximum, and the local reward $r(\cdot,s(\cdot))$ has positive total mass under every monotone predicate $s$.

\emph{Case 1: Separable objective.} Let $\bar r(S)=r(S,s_c(S))$. Because the $K$ proposals in a group are drawn i.i.d.\ from $q$, linearity of expectation yields
\begin{equation*}
J_{\mathrm{sep}}(q)=\sum_{i=1}^{K}\mathbb{E}_{S_i\sim q}\,\bar r(S_i)
=K\sum_{S\subseteq\mathcal E}q(S)\,\bar r(S),
\end{equation*}
which is linear in $q$. Let $\bar r^\star=\max_S\bar r(S)$ and $Q^\star=\arg\max_S\bar r(S)$. Then
\begin{equation*}
J_{\mathrm{sep}}(q)=K\bar r^\star-K\sum_{S}q(S)\,\big(\bar r^\star-\bar r(S)\big)
\le K\bar r^\star,
\end{equation*}
with equality if and only if $q(Q^\star)=1$. Thus the maximizers of $J_{\mathrm{sep}}$ are precisely the distributions supported inside $Q^\star$. In particular, any point mass concentrated on a single subset in $Q^\star$ is globally optimal. Whenever $|\Msf(c)|\ge2$, such a point-mass maximizer misses at least one member of $\Msf(c)$. Moreover, if $\bar r(A)>\bar r(B)$ for two minimal witnesses $A,B\in\Msf(c)$, then $B\notin Q^\star$ and every maximizer assigns zero probability to $B$. Note that this collapse holds for any per-proposal reward $\bar r$, whether local or not.

\emph{Case 2: Proportional law.} Choose three distinct elements $a,b,d\in\mathcal E$ and consider two monotone predicates $s_1$ and $s_2$ generated by the antichains $\Msf_1=\{\{a,b\},\{d\}\}$ and $\Msf_2=\{\{a\},\{d\}\}$, respectively. The subset $X=\{a,b\}$ is sufficient under both predicates ($s_1(X)=s_2(X)=1$), so locality assigns it the exact same unnormalized weight $r(X,1)$ under both $s_1$ and $s_2$. If the target distribution under $s_1$ fails to match $\Msf_1$, then $s_1$ already witnesses the claim. Otherwise, $X$ belongs to the support under $s_1$, which implies $r(X,1)>0$. Under $s_2$, the target distribution must then also assign positive probability mass to $X=\{a,b\}$, even though $X$ strictly contains the minimal witness $\{a\}$ and is therefore non-minimal under $s_2$.

\emph{Case 3: Success probability.} Let $q^\star$ maximize the objective and let $t=q^\star(F(c))$ be its marginal success probability. Whenever $|\Msf(c)|\ge2$, the empty set cannot be sufficient ($\varnothing\notin F(c)$), because $\varnothing\in F(c)$ would force $\Msf(c)=\{\varnothing\}$. Meanwhile, because $F(c)$ is a non-empty up-set, the full ground set $\mathcal E$ always belongs to $F(c)$, and because $|\Msf(c)|\ge2$, $\mathcal E$ strictly contains every minimal witness and is itself non-minimal. Consequently, the two-point distribution $q' = t\,\delta_{\mathcal E}+(1-t)\,\delta_{\varnothing}$ achieves the exact same success probability $q'(F(c))=t$ and the exact same objective value as $q^\star$, yet its support lies entirely within $\{\varnothing,\mathcal E\}$ and misses $\Msf(c)$ completely.
\end{proof}

Theorem~\ref{thm:collapse} extends immediately to practical policy-gradient estimators built around separable objectives. Subtracting an admissible baseline leaves the expected gradient of $J_{\mathrm{sep}}$ unchanged (Appendix~C.2), and group-level advantage normalization (as in GRPO) rescales each group's contribution by a positive statistic; both are used for variance control in the gradient estimator of $J_{\mathrm{sep}}$, and the theorem concerns the objective being estimated. Similarly, replacing the sum over a batch with a best-of-$K$ maximum $\max_i\rho(S_i)$ still satisfies $\mathbb{E}_q[\max_i\rho(S_i)]\le\rho^\star$, so a point mass on a single maximizer in $Q^\star$ remains globally optimal. When $\rho$ is the binary verifier bit $s_c(S)$, the expected best-of-$K$ return reduces to $1-(1-q(F(c)))^K$, which depends on $q$ solely through the success rate $q(F(c))$ and is indifferent among all distributions supported on $F(c)$. This places pass@$K$ policy optimization and Maximum Likelihood RL (MaxRL), which maximizes $\log q(F(c))=-\sum_{k\ge1}(1-q(F(c)))^k/k$, directly under the third case of Theorem~\ref{thm:collapse}. Likewise, Maximum-Entropy RL with per-proposal reward $r_0$ and temperature $\beta$ targets $q(S)\propto n(S)\,e^{r_0(S,s_c(S))/\beta}$ (where $n(S)$ counts the number of construction trajectories terminating at $S$), and a GFlowNet with terminal reward $R$ targets $q(S)\propto R(S)$. Because both weights are local functions of $(S, s_c(S))$, both methods fall under the second case of Theorem~\ref{thm:collapse}.

While Theorem~\ref{thm:collapse} characterizes the optima of separable objectives, our next result proves that the MWRL coverage return itself cannot be represented as any proposal-separable sum.

\begin{theorem}[Proposal-separable inexpressibility]
Let $K\ge2$, let $\mu$ have full support on $2^{\mathcal E}$, and suppose that the
sufficient region $F(c)$
contains a strict containment chain $A\subsetneq B$. Then the coverage return
\begin{equation*}
G_\mu(S_{1:K},y_{1:K})=
\mu\!\left(\bigcup_{i:y_i=1}\upset{S_i}\right)
\end{equation*}
is not proposal-separable.
\end{theorem}
\begin{proof}
Let $w_A=\mu(\upset{A})$ and $w_B=\mu(\upset{B})$. Because $\mu$ has full support and
$A\subsetneq B$, we have $\upset{B} \subsetneq \upset{A}$ and thus $w_A>w_B$. Now compare a batch consisting of $K$ identical copies of $A$ against a batch consisting of $K-1$
copies of $A$ and one copy of $B$. Because $\upset{B} \subset \upset{A}$, both batches certify the exact same region $\upset{A}$ and achieve total coverage $w_A$. If $G_\mu$ were proposal-separable, equating the two batch returns would force
$r_c(A,1)=r_c(B,1)$. However, a batch of $K$ copies of $B$ would then have to receive the same return as a
batch of $K$ copies of $A$, contradicting $w_B<w_A$.
\end{proof}

Because the coverage return $G_\mu$ depends on containment and overlap among the proposals in a batch, its marginal credit inherently couples the proposals within each group.

\subsection{Nonmonotone verifiers and existential closure}
\label{app:nonmono}
When a raw verifier $s$ is not strictly monotone, $s(S)=1$ confirms that $S$ itself is sufficient, but it no longer guarantees that every superset of $S$ will also be accepted by $s$. Even in this regime, the MWRL coverage return continues to penalize supersets and duplicates among accepted proposals, although the formal up-set union $U_K$ may include supersets where the raw verifier drops back to zero. Without full monotonicity, global subset-minimality can diverge from local \emph{1-minimality} (where an accepted set $S$ is 1-minimal if removing any single element causes the verifier to reject it). Under a monotone predicate, 1-minimality and subset-minimality coincide.

Monotonicity can be restored formally through the existential closure $s_\exists(S) = \mathbf{1}[\exists Q \subseteq S : s(Q)=1]$, which certifies all supersets of any accepted subset and preserves Proposition~\ref{prop:support}. Because our coverage credit requires only the raw verifier bit and terminal subset, our language-model circuit experiments train directly on the raw verifier $s_0$ without invoking any canonicalizer during training. To audit how often the raw circuit verifier violates monotonicity in practice, we define the \emph{finite closure} of a rejected superset as the output of our circuit canonicalizer $\Phi$ (Appendix~\ref{app:circuits}), which searches inside the rejected superset for an accepted subset. Throughout the circuit experiments, recovered circuits are 1-minimal under the raw verifier $s_0$.

\paragraph{Measured violation and repair.}
For each behavior, we extract verified circuits by applying $\Phi$ to the full component set and to random dense subsets. We then form supersets by adding random components to each verified circuit and query the raw verifier $s_0$ on those supersets; a strictly monotone verifier would accept $100\%$ of them. Table~\ref{tab:nonmono} reports the raw rejection rate, the fraction of rejected supersets repaired by the finite closure $\Phi$, and the residual violation rate, while Figure~\ref{fig:nonmono} plots rejection rates and faithfulness as a function of the number of added components. On Qwen3-1.7B, rejected supersets drift further below the $\alpha=0.8$ faithfulness threshold as more random components are added. On Qwen3-8B, monotonicity violations are less frequent and rejected supersets remain clustered near the $0.8$ boundary. Because the raw verifier can reject supersets of a valid circuit, pointwise acceptance does not certify an entire up-set under $s_0$; the residual row in Table~\ref{tab:nonmono} upper-bounds the violations that persist after finite-closure search.

\begin{table}[t]
\centering
\footnotesize
\setlength{\tabcolsep}{3pt}
\begin{tabular}{@{}lccc@{}}
\toprule
& Bench10 & 8B family & 8B leaf \\
\midrule
Probes & $630$ & $546$ & $2{,}688$ \\
Raw violation $\downarrow$ & $0.83$ & $0.56$ & $0.56$ \\
Closure repair $\uparrow$ & $0.70$ & $0.73$ & $0.68$ \\
Residual $\downarrow$ & $0.25$ & $0.15$ & $0.18$ \\
1-minimal & $15/15$ & $13/13$ & $64/64$ \\
\bottomrule
\end{tabular}
\caption{Monotonicity audit of the raw circuit verifier at threshold $\alpha=0.8$. Bench10 uses Qwen3-1.7B, while 8B family and 8B leaf use Qwen3-8B. Raw violation is the fraction of constructed supersets rejected by the raw verifier $s_0$; closure repair is the fraction of those rejected supersets within which the finite closure $\Phi$ successfully recovers an accepted subset. Residual bounds the remaining uncertified supersets after finite closure. The final row confirms that all circuits returned by $\Phi$ are 1-minimal.}
\label{tab:nonmono}
\end{table}

\begin{figure}[!t]
\centering
\includegraphics{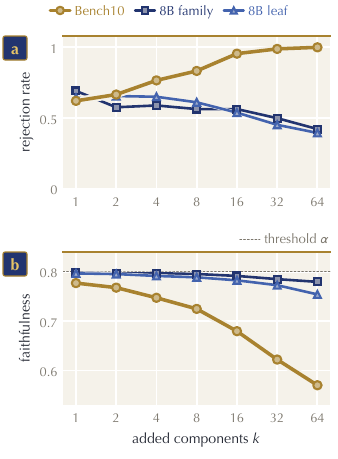}
\caption{Empirical monotonicity violations in the raw circuit verifier. \textbf{(a)}
Fraction of random supersets of a verified circuit rejected by the raw
predicate $s_0$ as a function of the number of added components (a strictly monotone predicate
would have zero rejections). \textbf{(b)} Mean distributional faithfulness of the rejected
supersets relative to the acceptance threshold $\alpha=0.8$ (dotted line).}
\label{fig:nonmono}
\end{figure}

\paragraph{Quality of the returned circuits.}
Every circuit recovered by the finite closure $\Phi$ passed full 1-minimality verification under the raw predicate $s_0$. The unrepaired residuals reflect misses of the attribution-guided greedy search rather than stochasticity, as re-running $\Phi$ on the same rejected supersets yields the same outcome. Notably, $\Phi$ is not constant across an up-set: among repaired supersets, $\Phi$ recovers the original seed circuit roughly half the time and lands on an alternative 1-minimal circuit in the remaining cases. Because the counting measure deduplicates proposals via $\Phi(S)$, sensitivity in $\Phi$ can alter which proposals are counted as distinct. By operating directly on raw terminal proposals and raw verifier bits, our coverage measure $\mathcal R_\mu$ avoids any dependence on $\Phi$ during training.

\section{Planning and finite-budget recovery}
\label{app:vi}

\subsection{Minimal-witness value iteration}
Throughout this section, let $\mathcal R=\mathcal R_\mu$ for a full-support base measure
$\mu$ on $2^{\mathcal E}$. Fix a context $c$ with sufficient region $F(c)=\upset{\Msf(c)}$. In our planning formulation, a state consists of the currently certified
region $F\subseteq F(c)$ paired with a remaining commitment budget $b$. Committing an accepted subset $S$ transitions the state to $F\cup\upset{S}$ and yields the immediate marginal coverage gain
$\Delta_{\mathcal R_\mu}(S\mid F)=\mu(F\cup\upset{S})-\mu(F)$. This defines a standard finite-horizon dynamic program over the augmented state space $(F,b)$; for any budget $b<|\Msf(c)|$, backward value iteration computes the coverage-optimal $b$-witness subset of $\Msf(c)$.

\begin{theorem}[Recovery by minimal-witness value iteration]
Let $s_c$ be monotone and let $\mu$ have full support on $2^{\mathcal E}$. Initialize
$F_0=\varnothing$ and apply the greedy minimal-witness update~\eqref{eq:vi} until the
maximum marginal gain reaches zero. Then every selected subset is a distinct, previously unselected member of
$\Msf(c)$, and the planner terminates after exactly $|\Msf(c)|$ steps with
\begin{equation*}
\{S_1,\ldots,S_{|\Msf(c)|}\}=\Msf(c),\qquad
F_{|\Msf(c)|}=F(c).
\end{equation*}
\end{theorem}
\begin{proof}
Proceed by induction on the step $k$, assuming all subsets selected prior to step $k$ are distinct members of $\Msf(c)$. For any unselected minimal witness
$m\in\Msf(c)$, the lattice point $m$ itself cannot belong to the up-set of any other minimal witness, so $m \notin F_k$ and its marginal gain satisfies
$\Delta_{\mathcal R_\mu}(m\mid F_k)\ge\mu(\{m\})>0$. Now consider any non-minimal accepted subset $S$, which must strictly contain at least one minimal witness $m\subsetneq S$ for some $m\in\Msf(c)$. If $m$ has already been selected in an earlier step, then $\upset{S}\subseteq\upset{m}\subseteq F_k$, giving $\Delta_{\mathcal R_\mu}(S\mid F_k)=0$. If $m$ has not yet been selected, then $\upset{S}\setminus F_k\subsetneq\upset{m}\setminus F_k$ because $m \in \upset{m}\setminus (F_k \cup \upset{S})$, which implies $\Delta_{\mathcal R_\mu}(S\mid F_k)<\Delta_{\mathcal R_\mu}(m\mid F_k)$ by full support of $\mu$. Consequently, the argmax in~\eqref{eq:vi} must select an undiscovered member of $\Msf(c)$ at every step. After $|\Msf(c)|$ steps, all minimal witnesses have been selected, their up-sets cover $F(c)$, and every remaining marginal gain is zero.
\end{proof}

On our enumerable MaxSAT instances, executing this dynamic program over $(F,b)$ provides the exact planning reference reported in Table~\ref{tab:maxsat} (see Appendix~\ref{app:vi-exec} for implementation and verification details).

\subsection{Finite-group objective identities}
Let $q_\theta$ denote the policy's terminal proposal distribution. For any lattice point $T \subseteq \mathcal E$, let
$\rho_\theta(T)=\Pr_{S\sim q_\theta}[s_c(S)=1,\ S\subseteq T]$ be the probability that a single rollout certifies $T$, and for each minimal witness $m \in \Msf(c)$, let
$p_m=\Pr_{S\sim q_\theta}[s_c(S)=1,\ \Phi(S)=m]$ be the probability that a single rollout canonicalizes to $m$. Because the $K$ rollouts in a group are drawn independently, the expected coverage and counting returns admit exact pass@$K$-style decompositions:
\begin{align}
J^{\mathrm{cov}}_K
&=\mathbb{E}_{T\sim\mu}\big[1-(1-\rho_\theta(T))^K\big],\label{eq:passk-cov}\\
J^{\mathrm{cnt}}_K
&=\sum_{m\in\Msf(c)}\big[1-(1-p_m)^K\big].\label{eq:passk-cnt}
\end{align}
Equation~\eqref{eq:passk-cov} shows that $J^{\mathrm{cov}}_K$ is the $\mu$-weighted average probability that a lattice point $T$ is certified by at least one of the $K$ proposals, while Equation~\eqref{eq:passk-cnt} sums the probability that each canonical witness $m$ appears at least once in the batch. For any group size $K \ge 2$, the derivative of $1-(1-p)^K$ with respect to $p$ is $K(1-p)^{K-1}$, which is strictly larger for smaller $p$, naturally prioritizing under-sampled witnesses. Of course, a witness assigned strictly zero probability ($p_m=0$) is never sampled on-policy, meaning sufficient exploration remains essential to ensure all regions of the antichain are reachable during training.

\subsection{The counting measure}
When a canonicalizer $\Phi$ is available to shrink any accepted proposal to a minimal witness, the counting measure scores a batch by the number of distinct canonical witnesses it recovers. Because verifier queries made inside $\Phi$ count against the total evaluation budget, this variant is practical primarily when $\Phi$ is computationally cheap and reliable. Given the accepted proposals $W=\{S_j:y_j=1\}$ in a batch, let $\Phi(W)=\{\Phi(S):S\in W\}$ denote the set of distinct canonical witnesses produced by the group. Because every member of $\Phi(W)$ is minimal, $\Phi(W)$ forms an antichain and coincides with the minimal elements of the up-set $U(\Phi(W))$. The counting valuation assigns unit weight to each distinct canonical witness:
\begin{equation}
\mathcal R_{\#}\big(U(\Phi(W))\big)=\big|\Phi(W)\big|.
\end{equation}
Note that $\mathcal R_{\#}$ is a function of the canonicalized region $U(\Phi(W))$ rather than the raw region $U(W)$, since a greedy canonicalizer $\Phi$ need not be invariant across supersets (Appendix~\ref{app:nonmono}). Letting $W_{-i}=\{S_j:j\ne i,\ y_j=1\}$, the resulting leave-one-out deletion credit reduces to a binary uniqueness indicator:
\begin{equation}
\begin{aligned}
A_i&=\mathcal R_{\#}\big(U(\Phi(W))\big)-\mathcal R_{\#}\big(U(\Phi(W_{-i}))\big)\\
   &=y_i\,\mathbf 1\!\big[\,\Phi(S_i)\neq\Phi(S_j)\ \ \forall\,j\neq i:\ y_j=1\,\big],
\end{aligned}
\end{equation}
which awards unit credit $A_i=1$ if and only if rollout $i$ succeeds and no other accepted rollout in the batch maps to the same canonical witness $\Phi(S_i)$. (If two rollouts map to the same witness, removing either one leaves $|\Phi(W)|$ unchanged, so both receive $A_i=0$.) This defines the l1o estimator for the counting measure; the l2o baseline of Appendix~C is used exclusively with the coverage measure.

Table~\ref{tab:objective-mc-ablation} compares the coverage and counting measures across the eight MaxSAT formulas. Coverage uses only the verifier bit and must learn minimality, reaching a $0.76$ born-minimal rate and two thirds of the antichain; the counting canonicalizer forces every output minimal and nearly saturates recall, at the cost of the extra verifier calls made by $\Phi$.

\begin{table}[tb]
\centering
\footnotesize
\setlength{\tabcolsep}{4pt}
\begin{tabular}{@{}lccc@{}}
\toprule
Measure & Interface &
\shortstack{minimal-output\\rate $\uparrow$} &
\shortstack{antichain\\recall $\uparrow$} \\
\midrule
Coverage $\mathcal R_\mu$ & $S,s_c(S),\mu$
& $0.758{\scriptstyle\pm}0.010$ & $0.667{\scriptstyle\pm}0.016$ \\
Counting & $S,s_c(S),\Phi(S)$
& $1.000{\scriptstyle\pm}0.000$ & $0.993{\scriptstyle\pm}0.007$ \\
\bottomrule
\end{tabular}
\caption{Comparison of coverage and counting measures on the eight MaxSAT instances (mean $\pm$ standard deviation over three seeds). The interface column lists the inputs required by each return. Minimal-output rate is the fraction of delivered witnesses that are minimal (the raw born-minimal rate for coverage, and $1.000$ by construction after canonicalization for counting).}
\label{tab:objective-mc-ablation}
\end{table}

\subsection{Finite-budget optimum of the counting measure}
\begin{theorem}[Finite-budget optimum of the counting measure]
Fix $L=|\Msf(c)|\ge1$, and assume the terminal policy class can realize any
categorical distribution over $\Msf(c)$ (with any residual mass placed on failed proposals). For any group size $K\ge2$, the counting objective~\eqref{eq:passk-cnt} has the unique maximizer
\begin{equation*}
p_m^\star=\frac{1}{L}\quad(\forall m\in\Msf(c)),\qquad p_\bot^\star=0,
\end{equation*}
achieving the optimal expected count
\begin{equation*}
\max_p J_K^{\mathrm{cnt}}=
L\left[1-\left(1-\frac{1}{L}\right)^K\right].
\end{equation*}
For $K=1$, any failure-free distribution ($p_\bot^\star=0$) is optimal.
\end{theorem}
\begin{proof}
Because $f_K(p)=1-(1-p)^K$ is strictly increasing on $[0,1]$, any optimal distribution must assign zero mass to failure ($p_\bot^\star=0$, so $\sum_{m\in\Msf(c)} p_m = 1$). For $K\ge2$ and $L>1$, $f_K(p)$ is strictly concave on $[0,1]$. Applying Jensen's inequality across the $L$ witnesses yields
$\frac{1}{L}\sum_{m\in\Msf(c)} f_K(p_m)\le f_K(1/L)$, with equality if and only if $p_m = 1/L$ for all $m\in\Msf(c)$. The case $L=1$ is trivial, and for $K=1$ the objective reduces to $\sum_m p_m = 1$, which is maximized by every failure-free distribution.
\end{proof}

Under the optimal uniform allocation $p_m^\star = 1/L$, a batch of $K$ proposals achieves expected antichain recall $1-(1-1/L)^K$ (and clearly no single batch of size $K<L$ can recover all $L$ witnesses simultaneously).

\subsection{Deployment-budget recovery}
\label{app:kinf-limit}
At test time, evaluating~\eqref{eq:passk-cnt} at a deployment budget of $B$ independent rollouts gives the exact expected number of distinct minimal witnesses recovered in $B$ draws:
$\sum_{m\in\Msf(c)}\big(1-(1-p_m)^B\big)$, where $p_m = q_\theta(m)$ is the frozen policy's probability of terminating at minimal witness $m$. On the enumerable MaxSAT lattice, we compute $p_m$ exactly via a single forward pass over the subset DAG and plot the analytical deployment curves alongside empirical prefix counts in Figure~\ref{fig:kmatched}. As $B\to\infty$, each curve asymptotes to the number of minimal witnesses in $\Msf(c)$ supported by the policy.

\subsection{Asymptotic recovery under full support}
\begin{theorem}[Asymptotic recovery under full support]
For a fixed terminal distribution $q$, let
$\mathcal{A}_q=\{S\in F(c):q(S)>0\}$ denote its accepted support and let $L=|\Msf(c)|$. Then
\begin{equation*}
\lim_{K\to\infty}J_K^{\mathrm{cov}}(q)=\mu(\upset{\mathcal{A}_q}).
\end{equation*}
If $\mu$ has full support on $2^{\mathcal E}$, this limit equals the maximum possible coverage $\mu(F(c))$ if and only if
$q(m)>0$ for every minimal witness $m\in\Msf(c)$. Furthermore, if
$\epsilon=\min_{m\in\Msf(c)}q(m)>0$, then for every finite $K$,
\begin{align*}
\mu(F(c))-J_K^{\mathrm{cov}}(q)&\le\mu(F(c))(1-\epsilon)^K,\\
\Pr[U_K=F(c)]&\ge1-Le^{-K\epsilon}.
\end{align*}
\end{theorem}
\begin{proof}
In~\eqref{eq:passk-cov}, the term $1-(1-\rho_q(T))^K$ converges to $1$ as $K\to\infty$ if and only if $\rho_q(T)>0$, which holds precisely when $T \in \upset{\mathcal A_q}$. Because $2^{\mathcal E}$ is finite, termwise convergence gives $\lim_{K\to\infty}J_K^{\mathrm{cov}}(q)=\mu(\upset{\mathcal{A}_q})$.
Under a full-support measure $\mu$, $\mu(\upset{\mathcal{A}_q}) = \mu(F(c))$ if and only if $\upset{\mathcal{A}_q} = F(c)$, which requires every minimal witness $m\in\Msf(c)$ to lie in $\upset{\mathcal{A}_q}$. By subset-minimality of $m$, the only sufficient subset of $m$ is $m$ itself, so $m \in \upset{\mathcal{A}_q}$ if and only if $q(m)>0$. Finally, when $\min_{m\in\Msf(c)} q(m) \ge \epsilon > 0$, every $T\in F(c)$ has $\rho_q(T) \ge \epsilon$ and remains uncertified after $K$ draws with probability at most $(1-\epsilon)^K$. Summing over $T\in F(c)$ yields the coverage gap bound, while a union bound over the $L$ minimal witnesses in $\Msf(c)$ gives $\Pr[U_K=F(c)] \ge 1 - L(1-\epsilon)^K \ge 1 - Le^{-K\epsilon}$.
\end{proof}

The exact same concentration argument establishes Proposition~\ref{prop:groupmin}:

\begin{groupminrestate}
Let $W_K$ denote the set of distinct accepted proposals in a batch of size $K$. If the policy assigns probability $q(m)\ge\epsilon>0$ to every minimal witness $m\in\Msf(c)$, where $L=|\Msf(c)|$, then the group-minimal successes match the ground-truth antichain, $\minop_{\subseteq}W_K=\Msf(c)$, with probability at least $1-Le^{-K\epsilon}$.
\end{groupminrestate}
\begin{proof}
Each minimal witness $m\in\Msf(c)$ is absent from the $K$ draws with probability at most
$(1-\epsilon)^K\le e^{-K\epsilon}$. By the union bound, all $L$ members of $\Msf(c)$ appear in $W_K$ simultaneously with probability at least $1-Le^{-K\epsilon}$. Conditioned on this event, every $S\in W_K$ is sufficient ($s_c(S)=1$) and therefore contains some $m\in\Msf(c) \subseteq W_K$. Because no member of $\Msf(c)$ can contain another sufficient set as a proper subset, every $m\in\Msf(c)$ is inclusion-minimal in $W_K$. Conversely, any non-minimal $S\in W_K \setminus \Msf(c)$ strictly contains some $m\in\Msf(c) \subseteq W_K$ and is therefore pruned by $\minop_{\subseteq}W_K$. Thus $\minop_{\subseteq}W_K=\Msf(c)$.
\end{proof}

Whenever a batch contains the complete antichain $\Msf(c)$, its group-minimal successes coincide with $\Msf(c)$, and every non-minimal accepted proposal in the batch receives strictly zero deletion credit (Proposition~\ref{prop:support}).

\section{Group policy gradients and control variates}
\label{app:pg}
In this section, we derive the leave-one-out (l1o) deletion credit and the leave-two-out (l2o) baseline directly from the grouped objective $J_K(\theta)$.

\subsection{Grouped objective and REINFORCE}
\paragraph{Admissible proposal processes.}
Our policy-gradient derivations hold for any episodic decision process that outputs terminal subsets under four mild conditions:
(i)~\emph{Finite horizon:} Every episode terminates within a finite number of steps, ensuring a finite trajectory space where expectations and gradients commute.
(ii)~\emph{Terminal verifier interface:} A deterministic readout map $\psi$ maps each completed trajectory $\tau$ to a terminal subset $S=\psi(\tau)\subseteq\mathcal E$; the verifier is queried once per episode on $S$, and the group return depends on trajectory $\tau_i$ solely through $S_i=\psi(\tau_i)$ and $y_i=s_c(S_i)$.
(iii)~\emph{Differentiable policy factorization:} Trajectory probabilities factorize as $P_\theta(\tau\mid c)=\rho(\tau)\prod_t\pi_\theta(a_t\mid x_t)$, where $\rho(\tau)$ collects $\theta$-independent transition factors and each action probability $\pi_\theta(a_t\mid x_t)$ is positive and differentiable with respect to $\theta$.
(iv)~\emph{Independent rollout groups:} The $K$ trajectories in each group are sampled i.i.d.\ from $P_\theta(\cdot\mid c)$.
Under conditions (i) and (ii), the expected objective $J_K(\theta)$ depends only on the induced terminal distribution $q_\theta(S)=\Pr_{\tau\sim P_\theta}[\psi(\tau)=S]$, invariant to the internal state representation or intermediate trajectory path. For a group of $K$ i.i.d.\ trajectories at context $c$, the group state is
\begin{equation*}
U_K=\bigcup_{i:\,y_i=1}\upset{S_i},\qquad
\mathcal R(\varnothing)=0,\ \ \mathcal R\ \text{monotone},
\end{equation*}
matching~\eqref{eq:upset-value}.

\paragraph{Reachability and expressiveness.}
Two additional conditions connect the trajectory-level policy $\pi_\theta$ to the terminal distribution theorems in Appendix~B:
First, every minimal witness must be reachable by the policy class, $\Msf(c)\subseteq\bigcup_\theta\operatorname{supp}q_\theta$; by~\eqref{eq:passk-cnt}, any unreachable witness has $p_m=0$ and permanently caps recall. Second, the closure of the realizable terminal distributions $\{q_\theta\}_\theta$ must contain all categorical distributions supported on $\Msf(c)$, ensuring that the uniform finite-budget optimum and full-support asymptotic recovery are realizable. In our experiments, MaxSAT episodes add elements sequentially starting from $\varnothing$, while Suzuki and circuit episodes start from $\mathcal E$ and prune elements (with add actions available for backtracking). Both parameterizations satisfy reachability and expressiveness whenever the horizon is large enough to reach all members of $\Msf(c)$. More generally, any policy that decides element inclusion in a fixed canonical order generates each subset along exactly one path, so the chain rule over the prefix tree realizes every law on the reachable subsets, in particular every law on $\Msf(c)$, and adding extra editing actions preserves this expressiveness. Autoregressive sequence models and factorized Bernoulli masks likewise satisfy conditions (i) to (iv) whenever their outputs parse to subsets.

\paragraph{Score-function gradient.}
Fix a context $c$. Because the $K$ trajectories in a batch are drawn i.i.d.\ from $P_\theta$, the expected group return is
\begin{equation*}
J_K(\theta)=\sum_{\tau_{1:K}}\Big(\prod_{i=1}^K P_\theta(\tau_i)\Big)\mathcal R(U_K).
\end{equation*}
Applying the log-derivative trick $\nabla_\theta p=p\,\nabla_\theta\log p$ to the product $\prod_{i=1}^K P_\theta(\tau_i)$ gives
\begin{align}
\nabla_\theta J_K
&=\sum_{\tau_{1:K}}\Big(\prod_{i=1}^K P_\theta(\tau_i)\Big)\mathcal R(U_K)\sum_{i=1}^K\nabla_\theta\log
P_\theta(\tau_i)\notag\\
&=\mathbb{E}\Big[\mathcal R(U_K)\sum_{i=1}^K\nabla_\theta\log P_\theta(\tau_i)\Big].
\label{eq:reinforce}
\end{align}
Because $\rho(\tau)$ does not depend on $\theta$, $\nabla_\theta\log P_\theta(\tau)=\sum_t
\nabla_\theta\log\pi_\theta(a_t\mid x_t)$, requiring no gradients through the black-box verifier. Equation~\eqref{eq:reinforce} is standard group-level REINFORCE: it is unbiased, but it weights every trajectory in the batch by the exact same total group return $\mathcal R(U_K)$, failing to isolate which rollout actually contributed coverage.

\subsection{Admissible baselines and leave-one-out credit}
\paragraph{Admissible control variates.}
For each rollout $i$, let $b_i$ be any baseline that does not depend on trajectory $\tau_i$ (though it may depend arbitrarily on context $c$ and the remaining $K-1$ trajectories $\tau_{-i}$). Conditioning on $\tau_{-i}$ and summing over $\tau_i$, we have
\begin{align}
\mathbb{E}\big[b_i\,\nabla_\theta\log P_\theta(\tau_i)\big]
&=\mathbb{E}_{\tau_{-i}}\Big[b_i\!\sum_{\tau_i}P_\theta(\tau_i)\,\nabla_\theta\log
P_\theta(\tau_i)\Big]\notag\\
&=\mathbb{E}_{\tau_{-i}}\Big[b_i\,\nabla_\theta\!\sum_{\tau_i}P_\theta(\tau_i)\Big]\notag\\
&=\mathbb{E}_{\tau_{-i}}[b_i\,\nabla_\theta 1]=0,
\label{eq:legal}
\end{align}
since $\sum_{\tau_i}P_\theta(\tau_i)=1$. Thus, subtracting $\sum_{i=1}^K b_i\nabla_\theta\log P_\theta(\tau_i)$ from~\eqref{eq:reinforce} leaves the expected policy gradient completely unchanged for any choice of $\tau_i$-independent baselines $\{b_i\}$.

\paragraph{The leave-one-out deletion credit (l1o).}
The natural baseline tailored to our group coverage objective is the coverage achieved when rollout $i$ is omitted from the group:
$b_i=\mathcal R(U_{-i})$, where
$U_{-i}=\bigcup_{j\ne i,\,y_j=1}\upset{S_j}$. Because $U_{-i}$ depends only on $\tau_{-i}$, $b_i$ is admissible by~\eqref{eq:legal}. Substituting $b_i=\mathcal R(U_{-i})$ into~\eqref{eq:reinforce} yields the unbiased MWRL policy gradient:
\begin{align}
\nabla_\theta J_K&=\mathbb{E}\Big[\sum_{i=1}^K A_i\sum_t\nabla_\theta\log\pi_\theta
(a_{i,t}\mid x_{i,t})\Big],\notag\\
A_i&=\mathcal R(U_K)-\mathcal R(U_{-i}),
\label{eq:mwrladv}
\end{align}
recovering the deletion credit of~\eqref{eq:advantage}.

\subsection{Leave-two-out baseline}
\label{app:loo2}
Because $\mathcal R$ is monotone, the raw l1o deletion credit $A_i$ in~\eqref{eq:mwrladv} is always non-negative ($A_i \ge 0$). To center the credits so that failed or redundant rollouts receive negative advantages without introducing gradient bias, we subtract a second admissible baseline constructed from the remaining $K-1$ rollouts. For each $j\ne i$, let
$U_{-i,-j}=\bigcup_{k\ne i,j,\,y_k=1}\upset{S_k}$ denote the region certified when both rollouts $i$ and $j$ are excluded, and define the leave-two-out (l2o) baseline as the average marginal contribution among the other $K-1$ proposals inside $\tau_{-i}$:
\begin{equation*}
b_i=\mathcal R(U_{-i})-\frac{1}{K-1}\sum_{j\ne i}\mathcal R\big(U_{-i,-j}\big).
\end{equation*}
The baseline-adjusted advantage is then $\widetilde A_i=A_i-b_i$. Because both $\mathcal R(U_{-i})$ and every $\mathcal R(U_{-i,-j})$ depend exclusively on $\{\tau_k\}_{k\ne i}$, $b_i$ is statistically independent of $\tau_i$. By~\eqref{eq:legal},
$\mathbb{E}[\widetilde A_i\,\nabla_\theta\log P_\theta(\tau_i\mid c)]=\mathbb{E}[A_i\,\nabla_\theta\log P_\theta(\tau_i\mid c)]$, so the l2o estimator preserves the exact expected gradient $\nabla_\theta J_K$.

\subsection{Support of deletion credit}
\begin{supportrestate}
If $\mathcal R=\mathcal R_\mu$ for a full-support measure $\mu$, then
\begin{equation}
A_i>0\iff y_i=1\ \text{and}\ \nexists\,j\ne i:\ y_j=1,\ S_j\subseteq S_i.
\label{eq:support}
\end{equation}
\end{supportrestate}
\begin{proof}
Rollout $i$ contributes $\upset{S_i}$ to the union when $y_i=1$ and $\varnothing$ otherwise, so $U_K=U_{-i}\cup\upset{S_i}$ if $y_i=1$ and $U_K=U_{-i}$ if $y_i=0$. Consequently, $A_i=y_i\cdot\mu(\upset{S_i}\setminus U_{-i})\ge0$.
Because $U_{-i}$ is an up-set, the set difference $\upset{S_i}\setminus U_{-i}$ is empty if and only if $S_i\in U_{-i}$, which by definition of $U_{-i}$ holds if and only if there exists $j\ne i$ with $y_j=1$ and $S_j\subseteq S_i$. Since $\mu$ assigns strictly positive mass to every non-empty subset of $2^{\mathcal E}$, $\mu(\upset{S_i}\setminus U_{-i})>0$ holds if and only if~\eqref{eq:support} is satisfied.
\end{proof}
Thus, positive deletion credit is awarded exclusively to non-duplicated group-minimal successes (Figure~\ref{fig:ledger}).

\begin{figure}[tb]
\centering
\includegraphics[width=\linewidth]{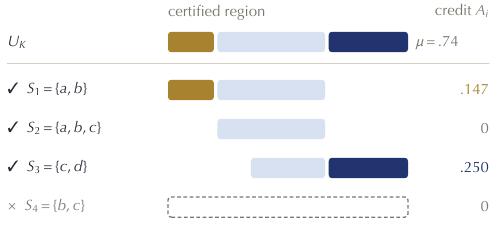}
\caption{Partition of certified measure and deletion credit for the example in Figure~\ref{fig:credit}, where the target minimal witnesses are $\{a,b\}$ and $\{c,d\}$. Horizontal bar widths are proportional to $\mu$-measure. The top bar displays the total group-certified region $U_K$, partitioned by which subset of proposals certifies each portion. Each lower bar displays the up-set $\upset{S_i}$ of an individual proposal $S_i$, where the solid colored segment represents $\upset{S_i}\setminus U_{-i}$ (the region certified exclusively by $S_i$, whose measure equals $A_i$). The minimal witnesses $S_1=\{a,b\}$ and $S_3=\{c,d\}$ receive positive credit, whereas the redundant superset $S_2=\{a,b,c\}$ and the rejected proposal $S_4=\{b,c\}$ receive $A_i=0$, illustrating Proposition~\ref{prop:support}.}
\label{fig:ledger}
\end{figure}

\section{Base-measure families and scalable credit computation}
\label{app:mc}

\subsection{Witness values}
\label{app:measures}

We first characterize cardinality-dependent base measures on $2^{\mathcal E}$ via their up-set witness values $\nu(S)=\mu(\upset{S})$.

\begin{rankrestate}
Suppose $\nu(S)=\mu(\upset{S})$ depends on $S$ only through its cardinality $|S|$. Then the single-element refinement ratio $\nu(|S|)/\nu(|S|{+}1)$ is constant across all cardinalities $|S|$ if and only if $\nu(S)=p^{|S|}$ for some parameter $p\in(0,1)$, which is uniquely realized by the product measure $\mu_p$ that includes each element $e \in \mathcal E$ independently with probability $p$.
\end{rankrestate}
\begin{proof}
Let $\nu(k)$ denote the witness value of any subset of cardinality $k$, normalized by
$\nu(0)=\mu(\upset{\varnothing})=\mu(2^{\mathcal E})=1$, and let $r_k=\nu(k)/\nu(k+1)$. If $r_k=r$ is constant across all $k$, then $\nu(k+1)=\nu(k)/r$, and induction yields $\nu(k)=r^{-k}=p^k$ with $p=1/r$. Conversely, $\nu(k)=p^k$ immediately gives a constant refinement ratio $r_k=1/p$. Full support and strict monotonicity of $\nu$ require $p\in(0,1)$. By the lemma on product measures below, the independent Bernoulli product measure with marginal inclusion probability $\Pr(t\in T)=p$ realizes $\nu(S)=\mu(\upset{S})=p^{|S|}$.
\end{proof}

In contrast, under polynomial reweightings of the layer-uniform measure, where $\nu(S)$ decays as an inverse power of $|S|$, the single-element refinement ratio $\nu(|S|)/\nu(|S|{+}1)$ shrinks toward $1$ as $|S|$ grows, weakening the relative incentive to prune larger sets (Figure~\ref{fig:measure-shrink}).

\begin{figure}[tb]
\centering
\includegraphics[width=\linewidth]{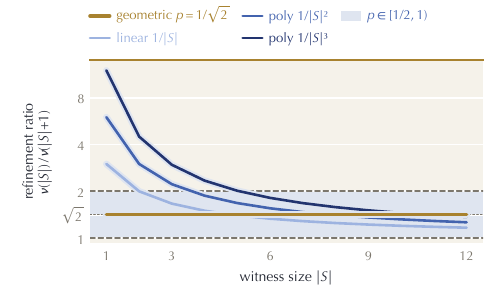}
\caption{Single-element refinement ratio $\nu(|S|)/\nu(|S|{+}1)$ across subset sizes $|S|$. The geometric base measure maintains a constant multiplicative refinement ratio $1/p$ across all cardinalities, whereas polynomial reweightings of the layer-uniform measure decay toward $1$ as $|S|$ increases. The shaded band spans the interval between uniform lattice weighting ($p=1/2$, strong size preference) and the size-neutral limit ($p\to 1$). Our default $p=0.7$ sits near $1/\sqrt2 \approx 0.707$, the geometric midpoint in log-ratio space.}
\label{fig:measure-shrink}
\end{figure}

\begin{lemma}[Product measures]
The uniform measure $\mu(T)=2^{-n}$ yields $\nu(S)=2^{-|S|}$. More generally, any independent Bernoulli product measure
$\mu=\bigotimes_{t\in\mathcal E}\mathrm{Bern}(p_t)$ yields $\nu(S)=\prod_{t\in S}p_t$; if
$\max_t p_t\le\bar p<1$, then $\nu(S)\le\bar p^{\,|S|}$ decays exponentially with $|S|$.
\end{lemma}
\begin{proof}
The event $T\supseteq S$ requires $t\in T$ for all $t\in S$ while leaving the remaining $n-|S|$ elements unconstrained. By coordinate-wise independence, $\Pr_{T\sim\mu}(T\supseteq S)=\prod_{t\in S}\Pr(t\in T)=\prod_{t\in S}p_t$.
\end{proof}

\begin{lemma}[Layer-uniform measure]
Suppose $T\subseteq\mathcal E$ is drawn by first sampling a cardinality $k\sim\mathrm{Unif}\{0,\dots,n\}$ and then sampling $T$ uniformly among all $\binom{n}{k}$ subsets of size $k$, so that $\mu(T)=
1/\big((n+1)\binom{n}{|T|}\big)$. Then $\nu(S)=1/(|S|+1)$.
\end{lemma}
\begin{proof}
Let $m=|S|$. Summing $\mu(T)$ over all supersets $T\supseteq S$ grouped by cardinality $k=|T|$ gives
$\nu(S)=\frac{1}{n+1}\sum_{k=m}^n\binom{n-m}{k-m}/\binom{n}{k}$.
Using the binomial identity $\binom{n-m}{k-m}/\binom{n}{k}=\binom{k}{m}/\binom{n}{m}$ and the hockey-stick identity $\sum_{k=m}^n\binom{k}{m}=\binom{n+1}{m+1}$, we obtain $\nu(S)=\frac{1}{n+1}\binom{n+1}{m+1}/\binom{n}{m}=1/(m+1)$.
\end{proof}

Both product measures and the layer-uniform measure assign strictly larger witness values to proper subsets, preserving the refinement order while differing in whether the decay with $|S|$ is exponential or inverse-linear.

\subsection{Exact inclusion--exclusion credit}
For a given accepted rollout $S_i$, let $D_i=\minop_\subseteq\{S_j:j\ne i,\ y_j=1\}$ denote the antichain of group-minimal successes among the remaining $K-1$ proposals. Using the up-set intersection identity $\upset{A}\cap\upset{B}=\upset{(A\cup B)}$, the principle of inclusion--exclusion gives
\begin{align*}
A_i&=\mu(\upset{S_i}\setminus U_{-i})\\
&=\nu(S_i)-\!\!\sum_{\varnothing\ne T\subseteq
D_i}\!\!(-1)^{|T|+1}\,\nu\Big(S_i\cup\textstyle\bigcup_{x\in T}x\Big),
\end{align*}
where each term $\nu(\cdot)$ is evaluated in closed form via Appendix~\ref{app:measures}. (If any $x\in D_i$ satisfies $x\subseteq S_i$, $A_i=0$ immediately by Proposition~\ref{prop:support}.) The exact sum costs $O(2^{|D_i|})$. Exact inclusion--exclusion is used when a group has at most twelve group-minimal successes.

\subsection{Shared-sample Monte Carlo credit}
For large groups, we avoid the exponential $O(2^{|D_i|})$ inclusion--exclusion expansion by noting that $\mathcal R_\mu(U_K)=\Pr_{T\sim\mu}[\exists i:\,y_i=1,\ T\supseteq S_i]$. Drawing a single shared batch of $N$ auxiliary subsets $T^{(1)},\dots,T^{(N)}\overset{\mathrm{iid}}{\sim}\mu$, the estimator
\begin{equation*}
\widehat A_i=\frac{1}{N}\sum_{m=1}^N\mathbf{1}\big[T^{(m)}\supseteq S_i\big]
\prod_{x\in D_i}\mathbf{1}\big[T^{(m)}\not\supseteq x\big]
\end{equation*}
is an unbiased estimate of $A_i = \mu(\upset{S_i}\setminus U_{-i})$. Computing $\widehat A_i$ requires $O(N|D_i|)$ bitwise subset checks and zero verifier queries, scaling linearly rather than exponentially in the number of group-minimal successes $|D_i|$.

\section{MaxSAT experiments}
\label{app:maxsat}

\subsection{Benchmark construction and ground truth}
Our MaxSAT benchmark provides an enumerable testbed where the complete ground-truth antichain $\Msf$ is known exactly. Each instance consists of $n$ Boolean variables $\mathcal E=\{1,\ldots,n\}$ and $m$ monotone positive clauses $C_1,\ldots,C_m\subseteq\mathcal E$. Proposing a subset $S\subseteq\mathcal E$ corresponds to setting the variables in $S$ to true and all variables in $\mathcal E \setminus S$ to false. The verifier evaluates whether $S$ satisfies every clause:
\begin{equation}
s(S)=\mathbf 1\!\left[\,S\cap C_j\ne\varnothing\ \text{for all}\ j\le m\,\right].
\label{eq:maxsat-verifier}
\end{equation}
Thus $s(S)=1$ if and only if $S$ is a hitting set of the clauses $\{C_1,\dots,C_m\}$, and the minimal witnesses $\Msf$ are the prime implicants of the monotone CNF formula (equivalently, the minimal hitting sets of its clauses).

We generate eight benchmark instances with $n=14$ variables, $m=10$ clauses, and $3$ variables per clause. Each clause is sampled uniformly without replacement, discarding duplicate clauses, using random seed $1000+i$ for instance $i$. We rejection-sample formulas to have between $6$ and $20$ minimal witnesses; the eight resulting instances have between $16$ and $19$ minimal witnesses (mean $|\Msf|=18.125$). For evaluation, we exhaustively test all $2^{14}=16{,}384$ subsets and extract the inclusion-minimal satisfying assignments to form the exact ground-truth antichain $\Msf$.

\subsection{Subset-construction MDP}
Each formula defines a finite-horizon episodic MDP over the subset lattice. At step $t$, the state is $(S_t,t)$, where $S_t \subseteq \mathcal E$ is the subset of variables selected so far. The policy receives a $15$-dimensional observation consisting of the $14$-bit indicator vector of $S_t$ and the normalized time step $t/H$. The available action space at state $S_t$ is
\begin{equation}
\mathrm{Act}(S_t)=\bigl(\mathcal E\setminus S_t\bigr)\cup\{\mathrm{STOP}\}.
\end{equation}
Selecting an unselected variable $e \in \mathcal E \setminus S_t$ deterministically transitions to $S_{t+1}=S_t\cup\{e\}$, while selecting $\mathrm{STOP}$ (or reaching the horizon $H=12$) terminates the episode at $S_T=S_t$. All intermediate step rewards are zero; upon termination, the verifier returns the single binary bit $s(S_T)$ from~\eqref{eq:maxsat-verifier}. Crucially, the environment neither halts automatically when $S_t$ first becomes sufficient nor prunes redundant variables. The policy must learn purely from terminal verifier feedback when to stop adding variables.

In the main comparison (Table~\ref{tab:maxsat}), a separate policy is trained from scratch on each formula with no formula identifier in the observation, isolating how each RL objective shapes the terminal proposal distribution on a fixed combinatorial landscape.

\subsection{Training, baselines, and evaluation}
All policy-gradient methods share an identical masked categorical actor architecture: a two-layer MLP with $128$ hidden units and $\tanh$ activations mapping the $15$-dimensional state observation to logits over the $14$ variable-addition actions and $\mathrm{STOP}$, with already-selected variables masked out prior to softmax normalization. The GFlowNet baselines use the same network architecture (Table~\ref{tab:maxsat-config}); the success and up-set GFlowNet rows of Table~\ref{tab:maxsat} train on $150$ updates of $16$ trajectories with a single step size of $3\times10^{-3}$.

\begin{table}[tb]
\centering
\small
\setlength{\tabcolsep}{4pt}
\begin{tabular}{@{}p{0.34\columnwidth}p{0.57\columnwidth}@{}}
\toprule
Item & Setting \\
\midrule
Instances & $8$ formulas, $14$ variables, $10$ clauses of length $3$ \\
Episode & add or $\mathrm{STOP}$, horizon $H=12$ \\
Actor & $15\!\to\!128\!\to\!128\!\to\!15$, $\tanh$ activations \\
Rollout batch & $8$ groups per update, $K=48$ proposals per group \\
Training & $150$ updates, $57{,}600$ terminal verifier calls per formula \\
Optimization & Adam, learning rate $3\times10^{-3}$, entropy weight $0.02$, max grad norm $1.0$ \\
MWRL & coverage measure $\mathcal R_\mu$ with geometric $\mu_p$ ($p=0.7$), l2o baseline,
mean-absolute rescaling, $N=20{,}000$ \\
Evaluation & $128$ fresh proposals per formula; policy seeds $0,1,2$ \\
\bottomrule
\end{tabular}
\caption{Shared hyperparameter configuration for the MaxSAT benchmark comparison.}
\label{tab:maxsat-config}
\end{table}

Because the verifier is queried exactly once per episode at the terminal subset $S_T$, every training run consumes $150 \times 8 \times 48 = 57{,}600$ verifier calls per formula, followed by $128$ evaluation rollouts. Auxiliary subsets drawn from $\mu_p$ for Monte Carlo credit estimation do not query the verifier, and ground-truth antichains are used exclusively for post-training evaluation metrics.

The proposal-separable RL baselines differ only in how they compute per-rollout advantages from the batch of verifier bits:
\begin{itemize}
\item \textbf{MaxRL} centers binary outcomes by the group mean and normalizes by the group mean.
\item \textbf{GRPO} centers outcomes by the group mean and normalizes by the group standard deviation.
\item \textbf{RLOO} subtracts the leave-one-out mean outcome of the other $K-1$ rollouts in the group.
\item \textbf{PPO} applies the standard clipped surrogate objective (clip ratio $0.2$) using MaxRL group advantages, while the remaining policy-gradient methods use unclipped score-function updates.
\item \textbf{PPO with a size penalty} applies the clipped PPO update to the group-standardized penalized reward $s(S)-0.1|S|$. Recovery has converged even when the policy returns more than one witness: extending training by a factor of four changes only one of twenty-four runs. The second witness appears on formulas whose two smallest witnesses have the same size. Their penalized rewards are equal, so a distribution over the tied pair is itself optimal.
\item \textbf{GFlowNet (success / up-set)} uses the trajectory balance loss with a uniform backward policy and terminal reward $R(S)=s(S)$ or $R(S)=s(S)p^{|S|}$, respectively.
\end{itemize}

For each frozen policy, we sample $128$ evaluation proposals and compute the born-minimal rate (fraction of accepted proposals that belong to $\Msf$), antichain recall (fraction of $\Msf$ recovered as distinct minimal witnesses across the $128$ proposals), the absolute number of distinct minimal witnesses recovered, and the median size of successful proposals. Metrics are averaged across the eight formulas for each seed, and we report the mean and standard deviation across the three seeds.

\subsection{Executed planning references}
\label{app:vi-exec}
Both planning reference rows in Table~\ref{tab:maxsat} are executed directly on the eight benchmark formulas. \textbf{Scalar VI} performs exact backward induction over the MDP state space $(S,t)$ with terminal reward
\begin{equation}
r_{\lambda}(S)=s(S)-\lambda|S|,
\end{equation}
using $\lambda=0.1$ and breaking ties deterministically in favor of the lowest-indexed variable. Rolling out the resulting optimal deterministic policy repeatedly returns the same single minimum-cardinality witness (and at $\lambda=0$, it halts at non-minimal supersets of mean size $8.0$).

\textbf{Minimal-witness VI} performs exact dynamic programming over states $(\mathcal C,b)$, where $\mathcal C\subseteq\Msf$ is the subset of minimal witnesses committed so far (generating the certified region $F=\upset{\mathcal C}$) and $b$ is the remaining commitment budget. For every $\mathcal C \subseteq \Msf$, we precompute $\mu_p(\upset{\mathcal C})$ via a fast subset-sum transform over witness-incidence masks and evaluate the Bellman recursion
\begin{equation}
V_b(\mathcal C)=\max\!\left\{\mathcal R_\mu(\upset{\mathcal C}),
\max_{M\in\Msf\setminus\mathcal C}V_{b-1}(\mathcal C\cup\{M\})\right\}.
\label{eq:maxsat-archive-vi}
\end{equation}
At budget $b=|\Msf|$, Minimal-witness VI recovers $100\%$ of the antichain across all eight formulas. We checked~\eqref{eq:maxsat-archive-vi} against brute-force enumeration of the selected sets for budgets $b \le 4$, and the within-episode Bellman recursion against complete trajectory enumeration on a separate $n=8$ formula; the maximum discrepancy is below $10^{-14}$.

\subsection{Main learned-policy comparison}
The prime-implicant comparison of Table~\ref{tab:maxsat} applies the preceding protocol independently to every formula and policy seed. We illustrate the terminal distributions induced by the three objective classes in Figure~\ref{fig:failure}(a) to (c) and the sizes of their recovered families on the eight formulas in Figure~\ref{fig:failure}(d). The complete antichain and both planning references make the terminal behavior of each objective directly observable.

\subsection{Credit estimator}
To isolate the bias and variance of the credit estimator in Figure~\ref{fig:bias}, we construct an enumerable $n=12$ formula with $8$ clauses, horizon $H=10$, and $|\Msf|=11$ minimal witnesses. We train an MWRL policy for $40$ updates, freeze its weights, and compute the exact analytical gradient $\nabla_\theta J_K(\theta)$ at group size $K=6$ by propagating exact terminal probabilities across all $2^{12}=4{,}096$ subsets and differentiating $J_K(\theta)$ in closed form.

We then sample $8{,}192$ rollout groups of size $K=6$ (partitioned into $16$ independent blocks of $512$ groups) and evaluate three advantage estimators on the exact same trajectories: raw non-negative deletion credit (l1o), leave-two-out baseline-adjusted credit (l2o), and naive group-mean centering ($A_i - \frac{1}{K}\sum_{j=1}^K A_j$). As shown in Figure~\ref{fig:bias}, both l1o and l2o converge to the exact gradient $\nabla_\theta J_K(\theta)$, whereas group-mean centering approaches a nonzero finite-$K$ offset. We therefore use l2o in the main MaxSAT comparison.

\subsection{Valuation and verifier access}
Table~\ref{tab:objective-mc-ablation} compares the coverage measure $\mathcal R_\mu$ against the counting measure $\mathcal R_{\#}$ on the eight MaxSAT formulas. Whereas $\mathcal R_\mu$ uses only the single terminal verifier bit $s(S_T)$ and relies on the up-set measure $\mu_p$ to reward pruning, $\mathcal R_{\#}$ invokes a greedy canonicalizer $\Phi$ on every accepted proposal $S_T$ (iteratively testing each element of $S_T$ in a deterministic order and deleting it if the remaining subset still satisfies all clauses, consuming up to $|S_T|$ extra verifier queries per accepted rollout). Thus, $\mathcal R_\mu$ places the burden of learning minimality on the policy gradient without extra verifier queries, whereas $\mathcal R_{\#}$ offloads minimality to deterministic post-processing at the cost of extra verifier calls.

\subsection{Training group size and deployment budget}
In Figure~\ref{fig:kmatched}, we decouple the training group size $K\in\{2,8,32,128\}$ from the test-time deployment budget $B$. Across six benchmark formulas and three seeds, we hold the total training budget strictly constant at $36{,}864$ verifier calls ($8$ groups per update for $2304$, $576$, $144$, and $36$ updates, respectively), ensuring larger groups $K$ receive no extra data or gradient steps.

After freezing each trained policy, we compute its exact terminal probability $p_M = q_\theta(M)$ on every minimal witness $M\in\Msf$ via forward propagation over the subset lattice and evaluate the expected number of distinct minimal witnesses recovered in $B$ test-time draws:
\begin{equation}
\mathbb E[N_B]=\sum_{M\in\Msf}\left[1-(1-p_M)^B\right],
\end{equation}
comparing this closed-form curve against empirical prefix counts from $1{,}024$ sampled rollouts for $B \in [1, 1024]$. Policies trained with tiny groups ($K=2$) concentrate their probability mass on a few witnesses, performing well at very small $B$ but plateauing early; policies trained with larger groups ($K=32, 128$) distribute mass across the full antichain, achieving much higher asymptotic recovery as $B$ increases.

\subsection{Antichain-geometry stress test}
\label{app:geometry}
To test how antichain geometry affects recovery, we replace random CNF clauses with directly planted antichains $\mathcal P \subseteq 2^{\mathcal E}$ over $n=20$ variables and run the monotone verifier
\begin{equation}
s_{\mathcal P}(S)=\mathbf 1[\exists M\in\mathcal P:M\subseteq S].
\end{equation}
We cross four structural factors in a $2\times 2\times 2\times 2$ factorial design ($16$ conditions, with $5$ planted instances and $3$ policy seeds per condition):
(1) antichain size $L\in\{8,32\}$;
(2) witness cardinality profile: balanced ($|M|=4$ for all $M\in\mathcal P$) vs.\ imbalanced (half of size $2$ and half of size $6$);
(3) pairwise element overlap: low or high, where the overlap levels target the $0.1$ and $0.9$ quantiles between the minimum and maximum attainable mean overlap for each size profile; and
(4) ratio of antichain size to group size $L/K\in\{0.5,2\}$.
In each condition, we train MWRL, size-penalized PPO, and an up-set GFlowNet for $128$ updates of $256$ rollouts ($32{,}768$ verifier queries, horizon $H=10$, $N=5{,}000$ Monte Carlo samples from $\mu_p$, and $512$ evaluation proposals). Figure~\ref{fig:geometry} reports the paired recall advantage of MWRL over the strongest baseline in each condition with $95\%$ bootstrap confidence intervals; the difference pairs MWRL with the stronger learned baseline within each condition after averaging policy seeds, and the confidence intervals bootstrap the five predicate-level paired means. MWRL's advantage is largest on large, balanced antichains ($L=32$, $|M|\equiv 4$), where many equal-sized minimal witnesses compete for probability mass and local rewards have no mechanism to prevent mode collapse or superset dilution. Size imbalance creates intrinsically favored witnesses and reduces the difference between objectives.

\begin{figure*}[t]
\centering
\includegraphics[width=\textwidth]{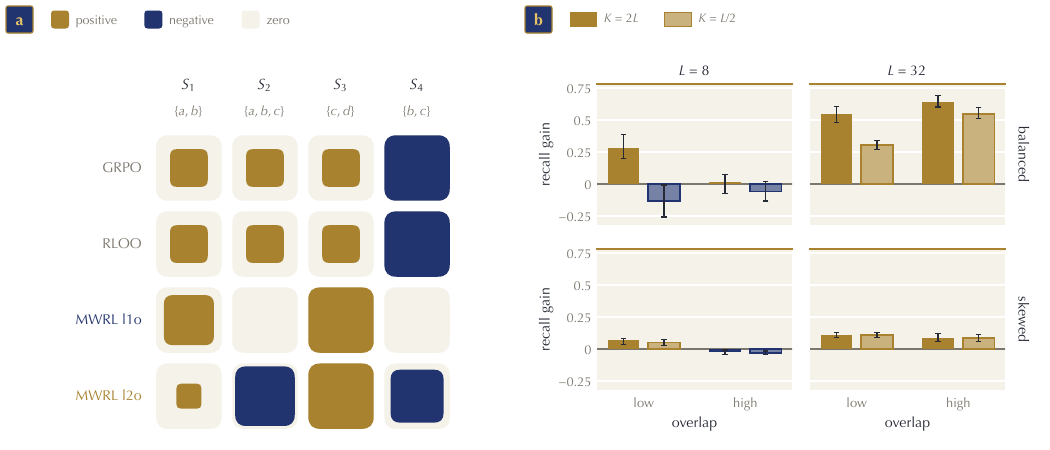}
\caption{Effect of planted antichain geometry on deletion credit and witness recovery. \textbf{(a)} Per-proposal advantages assigned within a four-rollout batch under the planted antichain $\mathcal P=\{\{a,b\},\{c,d\}\}$ ($p=0.7$). Because GRPO and RLOO rely solely on the binary verifier bit, they assign identical positive advantage to the minimal witnesses $S_1, S_3$ and to the redundant superset $S_2$, distinguishing only the failed rollout $S_4$. Raw deletion credit (l1o) assigns positive mass exclusively to the group-minimal witnesses ($0.147$ for $S_1$ and $0.250$ for $S_3$) and zero to $S_2$ and $S_4$. Subtracting the leave-two-out baseline (l2o) keeps $S_3$ and $S_1$ positive while pushing the redundant superset $S_2$ negative (even below the failed rollout $S_4$). \textbf{(b)} Paired antichain-recall gain of MWRL over the strongest baseline across $16$ planted geometry conditions after $32{,}768$ verifier calls ($5$ planted predicates $\times$ $3$ seeds per bar; error bars are $95\%$ bootstrap intervals; solid bars denote $K=2L$ and light bars $K=L/2$; gold marks a gain and lapis a loss). Columns are family sizes $L\in\{8,32\}$, rows balanced and skewed witness sizes, and within each panel the left and right pairs are the low-overlap and high-overlap quantiles.}
\label{fig:geometry}
\end{figure*}

\begin{table*}[tb]
\centering
\small
\setlength{\tabcolsep}{5pt}
\begin{tabular}{@{}lcccccc@{}}
\toprule
MaxEnt RL, $\beta$ & $\lambda=0$ & $0.1$ & $0.2$ & $0.3$ & $0.5$ & $1$ \\
\midrule
$0.02$ & 0.0 (0.00) & 3.2 (0.77) & 1.1 (0.87) & N/A & N/A & N/A \\
$0.05$ & 0.0 (0.00) & 0.4 (0.00) & 3.8 (0.49) & 0.3 (0.24) & 0.0 (0.00) & 0.0 (0.00) \\
$0.1$ & 0.0 (0.00) & 0.0 (0.00) & 0.4 (0.00) & \textbf{4.3 (0.12)} & 0.4 (0.20) & 0.0 (0.00) \\
$0.2$ & 0.0 (0.00) & 0.0 (0.00) & 0.0 (0.00) & 0.0 (0.00) & 2.0 (0.02) & 0.2 (0.13) \\
\bottomrule
\end{tabular}

\vspace{0.8em}
\begin{tabular}{@{}lcccccccc@{}}
\toprule
GFlowNet, $\lambda$ & $0.00$ & $0.36$ & $0.69$ & $0.90$ & $1.20$ & $1.50$ & $1.80$ & $2.30$ \\
\midrule
\# found & 0.2 & 0.8 & 2.0 & 3.0 & \textbf{4.3} & 3.6 & 3.1 & 3.0 \\
born-min.\ rate & 0.00 & 0.01 & 0.02 & 0.03 & \textbf{0.07} & 0.14 & 0.21 & 0.37 \\
\bottomrule
\end{tabular}
\caption{Hyperparameter sweeps over size-penalty strength $\lambda$ and entropy temperature $\beta$ for MaxEnt RL and GFlowNet on the eight MaxSAT formulas ($57{,}600$ verifier calls per formula, averaged over eight formulas and three seeds). MaxEnt RL cells report the number of distinct minimal witnesses recovered (\# found) followed by the born-minimal rate in parentheses. For GFlowNet, $\lambda=0$ and $\lambda=0.36=\log(1/0.7)$ correspond to the pure success and up-set rewards, respectively. Bold entries indicate the oracle best-performing settings reported in Table~\ref{tab:maxsat}; N/A denotes configurations not run.}
\label{tab:diversity-penalty}
\end{table*}

\subsection{Base-measure choice for coverage credit}
The base measure $\mu$ sets the relative witness values across cardinalities: $2^{-|S|}$ for the uniform measure, $1/(|S|+1)$ for the layer-uniform measure, and $p^{|S|}$ (with constant single-element refinement ratio $1/p$) for the geometric product measure $\mu_p$. We selected $p=0.7$ prior to running our benchmarks because $p=1/\sqrt{2}\approx 0.707$ sits at equal logarithmic distance between uniform lattice weighting ($p=1/2$, strong refinement pressure) and flat weighting ($p=1$, zero size preference), as illustrated in Figure~\ref{fig:pdefault}.

\begin{figure}[tb]
\centering
\includegraphics{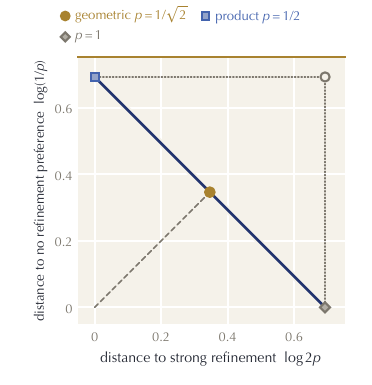}
\caption{Trade-off in log-ratio space between distance from strong refinement ($p=1/2$) and distance from size neutrality ($p=1$). The two logarithmic distances sum to $\log 2$ and are equal at $p=1/\sqrt2 \approx 0.707$, motivating our default choice $p=0.7$ (open circle marks the unattainable zero-trade-off origin).}
\label{fig:pdefault}
\end{figure}

Sweeping $p\in[0.5,0.9]$ on the MaxSAT benchmark (Figure~\ref{fig:psweep-main}) confirms that the born-minimal rate stays consistently high ($0.71$ to $0.79$) across the entire interval, because any full-support measure $\mu$ assigns positive deletion credit exclusively to group-minimal successes (Proposition~\ref{prop:support}). Meanwhile, increasing $p$ from $0.5$ to $0.9$ raises antichain recall from $0.55$ to $0.87$ (and shifts median proposal size from $3.8$ to $4.2$) by narrowing the witness-value gap between size-$4$ and size-$2$ minimal witnesses. (At the boundary $p=1$, all witness values collapse to $1$ and pruning incentive vanishes.)

\subsection{Credit-rescaling sensitivity}
Dividing a group's credits by a positive batch statistic rescales the step size per batch while preserving the sign and relative ratios of all advantages within the group. Table~\ref{tab:scale-ablation} compares dividing by the mean absolute raw credit against dividing by the standard deviation of the baseline-adjusted credit on the eight MaxSAT formulas (seed $0$, $K=48$, $57{,}600$ verifier calls). Under both group-mean centering and l2o, both rescaling rules achieve comparable antichain recovery; mean-absolute rescaling yields a higher born-minimal rate, while standard-deviation rescaling slightly increases recall under l2o.

\begin{table}[!htbp]
\centering
\footnotesize
\setlength{\tabcolsep}{2.5pt}
\begin{tabular}{@{}llccc@{}}
\toprule
Centering & Rescale & \shortstack{born-min.\\rate $\uparrow$} & \shortstack{antichain\\recall $\uparrow$} & \# found $\uparrow$ \\
\midrule
Group mean & Mean abs. & $0.879$ & $0.622$ & $11.25$ \\
Group mean & Std. dev. & $0.821$ & $0.615$ & $11.13$ \\
l2o & Mean abs. & $0.819$ & $0.675$ & $12.25$ \\
l2o & Std. dev. & $0.757$ & $0.687$ & $12.50$ \\
\bottomrule
\end{tabular}
\caption{Sensitivity of coverage MWRL to batch credit rescaling on the eight MaxSAT formulas (policy seed $0$).}
\label{tab:scale-ablation}
\end{table}

\subsection{Size-penalized maximum-entropy RL and GFlowNet}
\label{app:diversity-penalty}
To verify that local size penalties cannot rescue proportional samplers (Theorem~\ref{thm:collapse}), we perform an extensive hyperparameter sweep over size-penalty strength $\lambda$ for both Maximum-Entropy (MaxEnt) RL and GFlowNets on the eight MaxSAT formulas (Table~\ref{tab:diversity-penalty}). MaxEnt RL optimizes the entropy-regularized return~\citep{haarnoja2018}
\begin{equation}
G(\tau)=s(S)-\lambda|S|-\beta\sum_{t}\log\pi_\theta(a_t\mid x_t)
\end{equation}
for a trajectory $\tau$ terminating at $S$, using group standardization and clipped PPO updates with all other hyperparameters matched to Table~\ref{tab:maxsat-config} (at $\beta=0, \lambda=0.1$, this reduces to our size-penalized PPO baseline). GFlowNet optimizes trajectory balance~\citep{malkin2022} with a uniform backward policy and terminal reward $R(S)=s(S)\,e^{-\lambda|S|}$ (which reduces to the pure success and up-set GFlowNet baselines at $\lambda=0$ and $\lambda=\log(1/0.7)\approx 0.36$, respectively; failed proposals receive $R=\min(10^{-4},10^{-2}e^{-\lambda H})$, at least a hundredfold below any success within the horizon $H$). All runs use the full budget of $57{,}600$ verifier calls per formula ($3{,}600$ updates of $16$ trajectories for GFlowNet with learning rates $3\times 10^{-3}$ for the actor and $0.1$ for $\log Z$).

As shown in Table~\ref{tab:diversity-penalty}, both families show the same trade-off: small $\lambda$ leaves probability mass diffused over non-minimal supersets, whereas large $\lambda$ collapses the policy onto the smallest witnesses (or causes premature stopping before satisfying the clauses at $\lambda=1$). Recovery accordingly peaks at an intermediate strength in every row of Table~\ref{tab:diversity-penalty}. The peak moves to stronger penalties as $\beta$ grows, since the entropy term rewards longer trajectories. Without a penalty, the successful proposals of MaxEnt RL reach the horizon at every $\beta$, because the maximum-entropy optimum weights each set by the number of add orders that reach it, a count that the uniform backward policy of GFlowNet divides out. Even at the oracle peak of the grid ($\beta=0.1, \lambda=0.3$ for MaxEnt RL; $\lambda=1.20$ for GFlowNet), both methods recover an average of only $4.3$ minimal witnesses ($24\%$ recall) with born-minimal rates of $0.12$ and $0.07$, compared to $12.1$ witnesses ($67\%$ recall) and a $0.76$ born-minimal rate for MWRL.

\section{Scientific variable identification experiments}
\label{app:suzuki}

\subsection{Benchmark data and provenance}
Our Suzuki-Miyaura reaction benchmark combines real literature substrate pairs and condition vocabularies with complete combinatorial condition panels generated by a calibrated mechanistic reactivity simulator. During training and evaluation, the augmented reaction grid serves as a closed-world lookup table; the benchmark release also provides the raw empirical data partitions in their original schema.

The empirical foundation draws on two published sources:
\begin{enumerate}
\item The \textbf{Perera flow-chemistry screen}~\citep{perera2018}, which provides $5{,}760$ measured reactions across $15$ substrate pairs over a complete $12\times8\times4$ ligand, base, and solvent factorial grid.
\item The \textbf{ChemRANG literature catalog}~\citep{chemrang2026} (commit \texttt{dc0a14b}): out of $29{,}821$ raw records, $17{,}907$ pass structural and yield quality control, and deduplication of ambiguous or role-inconsistent entries yields $16{,}659$ canonical substrate pairs (each pairing a single-site aryl/heteroaryl halide or triflate with a single-site organoboron partner; rejected records and reason codes are archived in a quarantine table). Every task identity, substrate pair, and target product in our benchmark corresponds directly to a real published reaction.
\end{enumerate}

Because literature records report only isolated condition attempts rather than complete combinatorial panels, we complete the condition table for each substrate pair using a mechanistic response simulator modeling oxidative addition, transmetalation, reductive elimination, catalyst activation/deactivation, and major side reactions across $14$ discrete reaction-condition dimensions relative to a shared baseline protocol. For each task, the panel enumerates all condition assignments that deviate from the baseline in at most $3$ dimensions, averaging $4$ simulated replicates with noise of standard deviation $0.2$. Crucially, the mechanistic scoring rules reflect real catalytic trade-offs so that no single condition package dominates across all substrates: bulky electron-rich monophosphines favor hindered and aryl-chloride couplings but suffer chelation poisoning on unprotected heteroaryls, whereas small-phosphine and bidentate ferrocenyl systems favor heteroaryl routes.

Empirical data anchor the simulator in three ways: (i) ChemRANG condition frequencies define the discrete vocabulary and default baseline levels (without treating publication frequency as success probability); (ii) reported literature yields enter only as a shrinkage-regularized task-level activity offset proportional to the standardized median reported yield multiplied by $n/(n+2)$ for a pair with $n$ literature records; and (iii) the Perera factorial screen calibrates the yield distribution, its task heterogeneity, and the marginal and interaction structure of the factor effects; the response link is fitted on a fixed 32-task panel drawn before any outcome-conditioned selection, and the detection limit is set against simulated replicate means so that the zero-yield rate of the emitted table matches its calibration target. Every task retains explicit provenance links to its matching ChemRANG literature records; we emphasize that off-literature rows are calibrated mechanistic simulations rather than wet-lab measurements.

From the $16{,}659$ candidate tasks evaluated under generator seed $612$, $1{,}381$ pass our benchmark quality filters: an antichain size between $2$ and $80$ (with at most $10\%$ singletons), an overall sufficient-assignment fraction between $0.0005$ and $0.6$, exact agreement between the noiseless and replicate-averaged antichains, and a margin condition: for every minimal witness, its best assignment must exceed the yield threshold by at least one point, and every assignment that changes only a strict subset of its dimensions must stay at least one point below the threshold. From these $1{,}381$ valid tasks, a deterministic greedy selection followed by swap descent selects $1{,}024$ tasks that balance witness diversity and match pooled yield calibration bands while satisfying a strict popularity cap: a substrate-blind baseline that simply proposes the $16$ most frequent training witnesses must achieve no more than $0.45$ recall on held-out tasks (measuring $0.42$ on the final release; see Table~\ref{tab:suzuki-space}). The same variation across substrates occurs in the primary data: the thirteen substrate pairs of the Perera grid induce five distinct antichains, eight solved by a ligand swap alone, two with two alternative minimal witnesses, and three that require all three dimensions to move. Because the fingerprint and the generator use the same mechanistic factors, the conditioned-versus-blind gap measures how much the fingerprint improves recovery within this benchmark.

Table~\ref{tab:suzuki-space} summarizes the resulting benchmark environment: $1{,}024$ tasks, $14$ condition dimensions, and $2{,}879$ tabulated condition rows per task (baseline plus $28$ single, $346$ pair, and $2{,}504$ triple deviations, totaling $2{,}948{,}096$ reactions) evaluated at a global yield threshold of $75\%$. Ground-truth antichains range from $2$ to $35$ minimal witnesses per task (mean $8.6$), spanning $201$ unique minimal witnesses and $986$ distinct antichain combinations across the population. The Perera partition also defines a separate complete three-dimensional environment used for generator calibration and pilot runs, and no number in the paper comes from it.

\begin{table}[tb]
\centering
\small
\setlength{\tabcolsep}{4pt}
\begin{tabular}{@{}p{0.30\columnwidth}p{0.61\columnwidth}@{}}
\toprule
Item & Setting \\
\midrule
Tasks & $1{,}024$ substrate pairs ($819$ train, $205$ held out per seed) \\
Dimensions ($n=14$) & catalyst system (6), base (6), solvent (6), water fraction (3),
temperature (3), and binary flags for time, catalyst loading, base equivalents, boron
equivalents, concentration, ligand-to-Pd ratio, cosolvent, additive, and
atmosphere \\
Panel per task & $2{,}879$ rows: baseline, $28$ single, $346$ pair, and $2{,}504$ triple
deviations \\
Success criterion & reaction yield $\ge 75\%$ (global threshold) \\
Antichain statistics & $2$ to $35$ minimal witnesses per task (mean $8.6$); $201$ unique minimal
witnesses and $986$ unique antichains across tasks \\
Popularity audit & held-out recall of top-$16$ most frequent training witnesses is
$0.42$ (release cap $\le 0.45$); $0.15, 0.25, 0.63$ at budgets $4, 8, 32$ \\
\bottomrule
\end{tabular}
\caption{Summary of the augmented Suzuki-Miyaura reaction benchmark. Numbers in parentheses indicate the number of discrete levels per dimension (including baseline).}
\label{tab:suzuki-space}
\end{table}

\subsection{Verifier, decision process, and policy}
Opening a subset of condition dimensions $S \subseteq \mathcal E$ permits the reaction protocol to vary any dimensions inside $S$ while holding all dimensions in $\mathcal E \setminus S$ fixed at their baseline defaults. The verifier returns $s_c(S)=1$ if and only if the reaction table for substrate pair $c$ contains at least one assignment whose modified dimensions lie within $S$ that achieves $\ge 75\%$ yield (unlisted assignments count as failures). By construction, this closed-world subspace lookup is deterministic and monotone, serving as a fast, reproducible proxy for an inner Bayesian optimization loop restricted to the subspace $S$. No Bayesian-optimization loop is executed during the reported training or evaluation.

We formulate subset selection as a $14$-step sequential MDP starting from the full set $\mathcal E$, where the policy can add, remove, or stop on the $14$ condition dimensions. At each step, the policy observes the $14$-bit active dimension mask, the normalized step index, and a $33$-dimensional molecular fingerprint of the substrate pair computed from the halide and organoboron SMILES strings (capturing leaving-group class, organoboron class, heteroaryl and N-H content, ortho steric load around each coupling carbon, polar surface area, molecular weight, and lipophilicity proxies). A single conditional policy is trained across all $819$ training substrates using the MWRL coverage return ($\mathcal R_\mu$ with $p=0.7$ and l2o baseline; see Table~\ref{tab:suzuki-config}).

\begin{table}[tb]
\centering
\small
\setlength{\tabcolsep}{4pt}
\begin{tabular}{@{}p{0.30\columnwidth}p{0.61\columnwidth}@{}}
\toprule
Item & Setting \\
\midrule
Context fingerprint & $33$ mechanistic descriptors extracted from reactant SMILES \\
Actor architecture & hidden width $256$, $\tanh$; add, remove, or stop actions; full-set start; horizon $14$ \\
MWRL objective & coverage measure $\mathcal R_\mu$ with geometric $\mu_p$ ($p=0.7$), l2o baseline,
mean-absolute rescaling \\
Amortized training & $K=32$, $8$ task groups per update, $4{,}800$ updates \\
Per-substrate ceiling & $K=32$, $400$ updates trained separately per substrate \\
Optimization & Adam, learning rate $5\times10^{-4}$, entropy weight $0.02$, max grad norm
$1.0$ \\
Evaluation & $128$ frozen-policy proposals per task; seeds $0, 1, 2$ \\
\bottomrule
\end{tabular}
\caption{Hyperparameter configuration for the Suzuki-Miyaura reaction-condition discovery experiments.}
\label{tab:suzuki-config}
\end{table}

\subsection{Protocol and evaluation}
We run three random seeds, where each seed controls both the $819/205$ train/held-out substrate split and policy weight initialization (Table~\ref{tab:amortization}). All four compared methods share the exact same environment, verifier, and evaluation protocol:
\begin{itemize}
\item \textbf{Amortized MWRL (Conditioned)} receives the $33$-dimensional substrate fingerprint.
\item \textbf{Amortized MWRL (Substrate-blind)} optimizes the exact same MWRL objective on the training split, but has the substrate fingerprint zeroed out.
\item \textbf{Scalar RL} trains the fingerprint-conditioned actor using a standard best-of-group binary success credit in place of MWRL's up-set deletion credit.
\item \textbf{Per-substrate ceiling} trains a separate tabula-rasa MWRL policy for $400$ updates on each individual substrate pair (one entry of that row is a complete training run, so the ceiling is measured on a fixed sample of $256$ of the $1{,}024$ substrates at a single seed; the sample gives its mean a $95\%$ interval of $\pm 0.04$ antichain recall).
\end{itemize}
At evaluation time, we draw rollouts from each frozen policy in order, deduplicate them as sets, and evaluate a window of $32$ distinct condition sets per substrate pair; the window closes once $32$ distinct sets have been attempted, whether or not they verify. As in our MaxSAT coverage experiments, only raw terminal proposals that are themselves minimal count as recovered witnesses. Antichain recall is the recovered fraction of the task's ground-truth antichain within the window. The born-minimal rate is the fraction of successful attempts in the window whose raw terminal is already minimal, and the witness count is the number of distinct minimal witnesses recovered in the window. Recall is reported on training substrates and on held-out substrates; the born-minimal rate and the witness count are reported on the held-out split for the amortized policies and over the sampled substrates for the per-substrate ceiling. The evaluation stores $128$ rollouts per task in proposal order with their verification outcomes, so every statistic is reproducible at any other budget from the released artifacts; elapsed training time for each method is stored with them.

\subsection{Amortization on the fully synthetic library}
Prior to constructing the literature-anchored benchmark above, we also evaluated the same four methods on an end-to-end synthetic Suzuki library of $172$ substrate pairs ($138$ train / $34$ held-out per seed) over $18$ condition dimensions generated by our first-generation simulator. As shown in Table~\ref{tab:suzuki-synthetic}, the synthetic benchmark produces the exact same qualitative ranking: fingerprint-conditioned MWRL ($0.83$ held-out recall) nearly matches the per-substrate training ceiling ($0.88$), more than doubles the recall of the substrate-blind ablation ($0.38$), and cleanly outperforms scalar RL ($0.00$).

\begin{table}[H]
\centering
\footnotesize
\setlength{\tabcolsep}{3pt}
\begin{tabular}{@{}lcccc@{}}
\toprule
Method & born-min. & \multicolumn{2}{c}{recall {\scriptsize$\uparrow$}} & \#found {\scriptsize$\uparrow$} \\
\cmidrule(lr){3-4}
       & rate {\scriptsize$\uparrow$} & seen & held-out & \\
\midrule
Per-substrate & $0.58{\scriptstyle\pm}0.01$ & $0.88{\scriptstyle\pm}0.00$ & - & $6.4{\scriptstyle\pm}0.0$ \\
\midrule
\multicolumn{5}{@{}l}{\emph{Amortized MWRL}} \\
\quad Conditioned & $0.58{\scriptstyle\pm}0.03$ & $0.86{\scriptstyle\pm}0.01$ & $0.83{\scriptstyle\pm}0.04$ & $6.1{\scriptstyle\pm}0.3$ \\
\quad Substrate-blind & $0.11{\scriptstyle\pm}0.05$ & $0.41{\scriptstyle\pm}0.10$ & $0.38{\scriptstyle\pm}0.08$ & $2.7{\scriptstyle\pm}0.5$ \\
\midrule
Scalar RL & $0.00{\scriptstyle\pm}0.00$ & $0.00{\scriptstyle\pm}0.00$ & $0.00{\scriptstyle\pm}0.00$ & $0.0{\scriptstyle\pm}0.0$ \\
\bottomrule
\end{tabular}
\caption{Amortized recovery on the supporting $172$-substrate synthetic library (mean $\pm$ standard deviation over three seeds).}
\label{tab:suzuki-synthetic}
\end{table}

\FloatBarrier

\section{LLM sparse-circuit discovery}
\label{app:circuits}
\begin{figure*}[!t]
\centering
\includegraphics{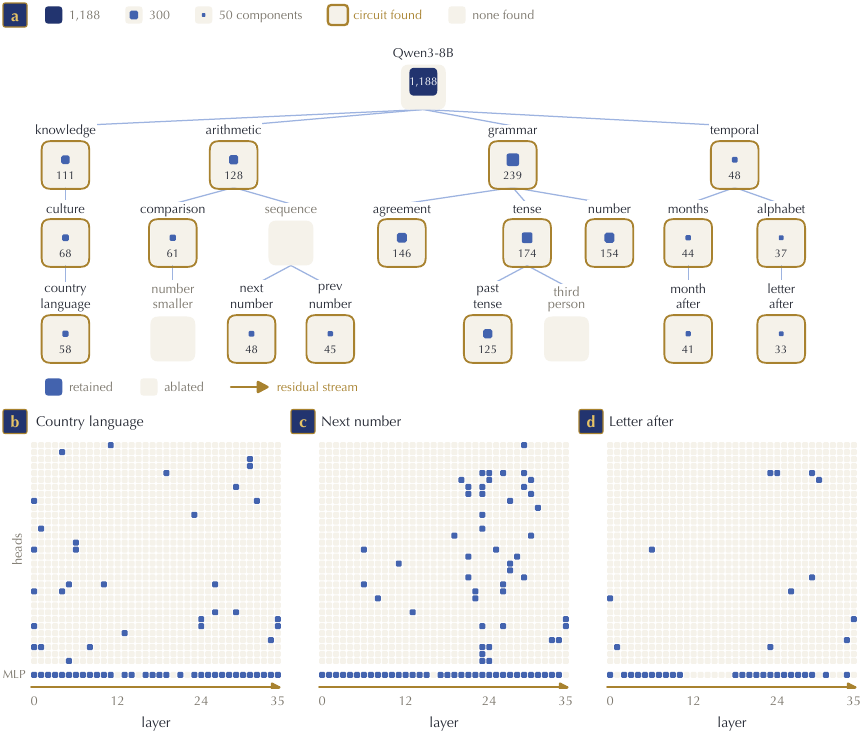}
\caption{Hierarchical sparse circuit discovery on Qwen3-8B. \textbf{(a)} Recovered portion of
the three-level behavior taxonomy. The number in each node is the median cardinality of its recovered 1-minimal circuits, and the tessera's area is proportional to it; gold rims mark nodes with a recovered circuit, and empty nodes tasks where no valid circuit was found within the evaluation budget. All recovered circuits reach a distributional faithfulness of about $0.80$. \textbf{(b to d)} Component supports for three leaf behaviors plotted on the Qwen3-8B module grid ($1{,}188$ components). Each leaf support is the
union of accepted canonical circuits across $24$ evaluation proposals, retaining
$63$, $81$, and $36$ modules, respectively (lapis tesserae mark retained attention heads and MLPs; empty cells mark ablated components).}
\label{fig:hiertree}
\end{figure*}

\begin{figure*}[tp]
\centering
\includegraphics[width=\textwidth]{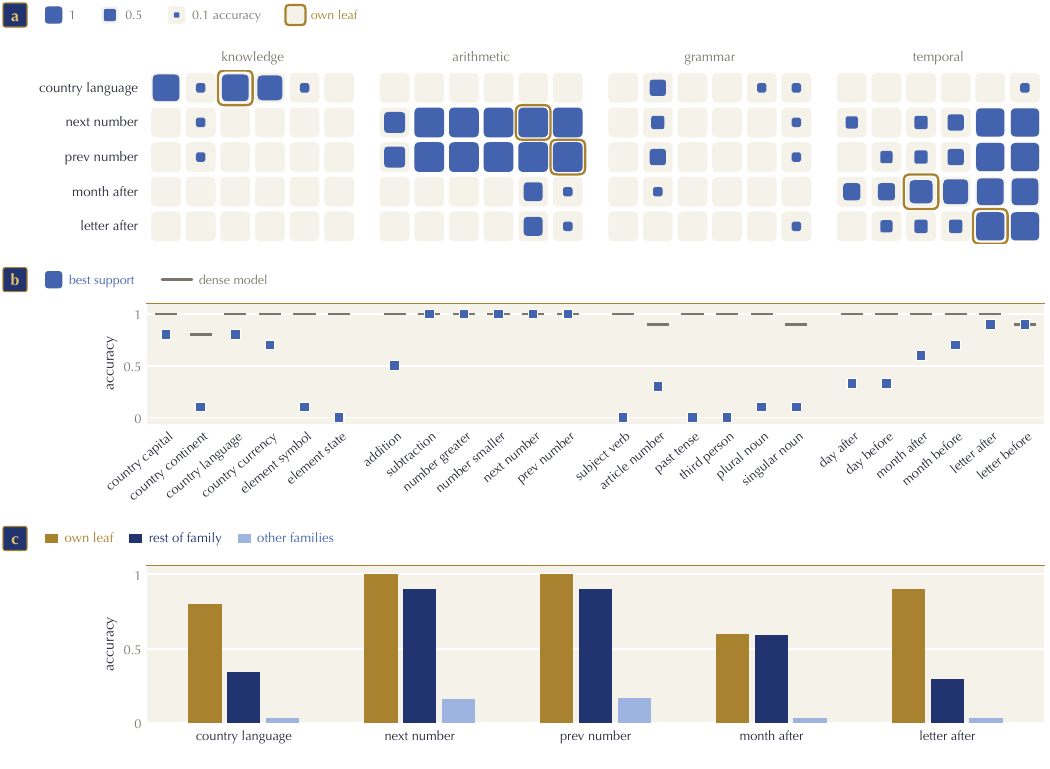}
\caption{Cross-task transfer and specialization of recovered Qwen3-8B leaf supports. \textbf{(a)} First-token
correct-versus-foil accuracy when activating each recovered leaf support in isolation (with all outside components patched with counterfactual activations) across every evaluation leaf task. Columns group leaves into the four top-level families, a tessera's area is the accuracy, empty cells indicate zero accuracy, and gold rims mark each support's own target leaf.
\textbf{(b)} Best accuracy achieved by any recovered support on each leaf vs.\ dense Qwen3-8B accuracy; the search recovered no grammar support.
\textbf{(c)} Mean accuracy of each support evaluated on its own target leaf ($0.86$), on sibling leaves within the same task family, and on leaves outside its family ($0.08$).}
\label{fig:hierconfusion}
\end{figure*}

In both language-model experiments (Qwen3-1.7B and Qwen3-8B), the transformer weights remain completely frozen and we search over subsets of attention heads (individual query-head slices feeding into the attention output projection) and MLP blocks (the full MLP output of a layer). For each clean probe prompt $x_i$, we pair it with a counterfactual prompt $\widetilde x_i$ drawn from another example of the same task. Ablating a component outside the retained subset $S$ replaces its clean activation on $x_i$ with its aligned activation recorded on $\widetilde x_i$.

Let $P_i$ denote the unmodified model's next-token distribution on prompt $x_i$, $P_{S,i}$ the distribution when only components in $S$ keep their clean activations, and $P_{\varnothing,i}$ the fully ablated distribution when every component receives its counterfactual activation. Across a probe set of $P$ prompts, define the baseline clean-to-corrupted KL divergence
\begin{equation}
\mathrm{KL}_{\varnothing}=\frac{1}{P}\sum_{i=1}^{P}
D_{\mathrm{KL}}\!\left(P_i\,\Vert\,P_{\varnothing,i}\right)
\label{eq:circuit-kl-empty}
\end{equation}
and the patched circuit divergence
\begin{equation}
\mathrm{KL}(S)=\frac{1}{P}\sum_{i=1}^{P}
D_{\mathrm{KL}}\!\left(P_i\,\Vert\,P_{S,i}\right).
\end{equation}
The raw circuit verifier accepts a proposed component subset $S$ when it recovers at least a fraction $\alpha=0.8$ of the clean model's distributional behavior:
\begin{equation}
s_0(S)=\mathbf 1\!\left[\mathrm{KL}(S)
\le (1-\alpha)\mathrm{KL}_{\varnothing}\right],\qquad \alpha=0.8.
\label{eq:circuit-verifier}
\end{equation}
We report normalized distributional faithfulness as $1-\mathrm{KL}(S)/\mathrm{KL}_{\varnothing}$. When a canonicalizer $\Phi$ is used (for the monotonicity audit in Appendix~\ref{app:nonmono} and support extraction in Qwen3-8B), $\Phi$ performs an EAP-ordered greedy pruning pass within $S$ where every deletion is verified by $s_0$, guaranteeing that $\Phi(S)$ is accepted by $s_0$ and 1-minimal.

\subsection{Qwen3-1.7B MMLU Bench10 and matched baselines}
\label{app:bench10}

\paragraph{Benchmark construction.}
To construct Bench10, we evaluate the dense Qwen3-1.7B model across all $57$ MMLU subjects (\texttt{cais/mmlu} test split) using zero-shot four-choice prompts ending with \texttt{Answer:}, and retain the ten subjects where the dense model prefers the correct answer token over the foil token on at least $92\%$ of discovery probes: marketing, astronomy, computer security, high-school biology, high-school chemistry, high-school computer science, international law, jurisprudence, management, and US foreign policy.

For each subject, the discovery split consists of $12$ questions (seed $0$, length $\le 700$ characters, stratified with $3$ questions per answer letter A, B, C, D), where the next question in cyclic order supplies the counterfactual activation prompt $\widetilde x_i$ and foil token. The held-out evaluation split selects $24$ disjoint questions per subject (seed $1$, $6$ per answer letter). Masks and baseline rankings are fixed from the discovery split before the held-out questions are loaded. A new verifier is then constructed on the held-out prompts, which recomputes $\mathrm{KL}_{\varnothing}$ on the held-out prompts to evaluate generalization.

\paragraph{Model and subset MDP.}
Qwen3-1.7B comprises $28$ layers with $16$ query heads and $1$ MLP per layer, yielding a ground set of $n = 28\times16+28=476$ components. We train a single conditional policy across all ten subjects. The policy's context vector is the subject's signed component-level Edge Attribution Patching (EAP) vector normalized by its maximum absolute entry, concatenated with the $476$-bit active component mask and the normalized step index ($953$ input dimensions in total). Starting from the full set $\mathcal E$ with horizon $H=476$, the actor (two $512$-unit $\tanh$ layers) selects single-component add, remove, or stop actions. Training uses our coverage return $\mathcal R_\mu$ ($p=0.7$, l2o baseline, mean-absolute rescaling) directly on the raw verifier $s_0$ without applying $\Phi$, running for $40$ updates of $K=12$ rollouts per subject ($4{,}800$ total verifier queries; see Table~\ref{tab:bench10-config}). At evaluation time, we sample $128$ proposals per subject from the frozen policy, deduplicate the accepted proposals, and verify 1-minimality via single-component deletion checks.

\begin{table}[tb]
\centering
\small
\setlength{\tabcolsep}{4pt}
\begin{tabular}{@{}p{0.34\columnwidth}p{0.57\columnwidth}@{}}
\toprule
Item & Setting \\
\midrule
Model & Qwen3-1.7B, bfloat16, eager attention, frozen weights \\
Components ($n=476$) & $448$ query heads and $28$ MLP blocks across $28$ layers \\
Verifier & raw faithfulness verifier $s_0$ at threshold $\alpha=0.8$ \\
Actor architecture & $953\!\to\!512\!\to\!512\!\to\!953$, $\tanh$ activations \\
Episode structure & add, remove, or stop; full-set start; horizon $476$ \\
Group return & coverage measure $\mathcal R_\mu$ ($p=0.7$), l2o baseline, mean-absolute scaling \\
Training budget & $K=12$, $10$ subject groups per update, $40$ updates \\
Optimization & Adam, learning rate $10^{-3}$, entropy weight $0.02$, max grad norm $1.0$ \\
Evaluation & $128$ frozen-policy proposals per subject; seed $0$ \\
Batching & $576$ prompt-mask sequences per forward pass on one NVIDIA L40 GPU \\
\bottomrule
\end{tabular}
\caption{Configuration for the Qwen3-1.7B MMLU Bench10 circuit discovery experiments.}
\label{tab:bench10-config}
\end{table}

\paragraph{Sparsity-matched baselines.}
For each subject, the smallest circuit recovered by MWRL among its $128$ evaluation proposals defines the target component budget $k$ for all ranking and mask baselines:
\begin{itemize}
\item \textbf{EAP}~\citep{syed2024} selects the top $k$ components by absolute attribution score.
\item \textbf{Magnitude} ranks attention heads by their output-projection weight blocks and MLPs by their down-projection weights, and selects the top $k$ components.
\item \textbf{Wanda}~\citep{sun2024} selects the top $k$ components by mean residual activation norm on the discovery prompts.
\item \textbf{HardConcrete}~\citep{louizos2018} optimizes a differentiable component mask for $300$ steps and retains the top $k$ components in the resulting mask.
\item \textbf{Random-$k$} averages across five uniformly sampled $k$-component subsets.
\item \textbf{ACDC (greedy)}~\citep{conmy2023} starts from the full component set, tests deletions in EAP order, and retains its own converged size.
\end{itemize}
The comparison matches sparsity; search cost differs by method and is not equalized.

\paragraph{Evaluation.}
Discovery-side evaluation reports the born-minimal rate, the number of distinct 1-minimal circuits, the successful-proposal rate, median circuit size, and mean distributional faithfulness. The held-out evaluation reports four-choice accuracy, correct-versus-foil accuracy, distributional faithfulness, and the verifier bit for every frozen mask. For MWRL it also reports the fraction of the recovered circuit family that remains sufficient and, for a representative circuit that remains sufficient, the fraction of its components whose deletion breaks sufficiency. The cross-subject matrix activates subject $j$'s representative circuit and measures four-choice accuracy on subject $i$'s $24$ held-out questions. Means and MWRL-minus-baseline differences in four-choice accuracy use $10{,}000$ paired bootstrap resamples over subjects.

While Table~\ref{tab:circuits} reports held-out distributional faithfulness ($\mathrm{KL}_{\text{held}}$) and discovery-split accuracy retention, we note that raw four-choice accuracy on held-out MMLU questions is low for all methods: at the matched budget a mask keeps roughly a tenth of the components, and answer accuracy measures performance of the full question-answering pipeline, most of which is ablated; the verifier tests distributional faithfulness, a different quantity. The $24$ held-out questions per subject are stratified by answer key, six per option, and chance is $0.25$. Accuracy ranges from $0.13$ for Random-$k$ and HardConcrete to $0.21$ for MWRL, compared with $0.64$ for the dense model. On this balanced answer key, the below-chance scores indicate a systematic preference for a plausible foil. This preference is as strong for random masks as for discovered circuits, pointing to heavily ablated decoding as its source. Four-choice accuracy does not reliably distinguish the methods at this budget: the MWRL-minus-EAP difference is $+0.02$, with a $95\%$ paired bootstrap interval covering zero. Table~\ref{tab:circuits} therefore reports distributional faithfulness, which does separate the method families on held-out questions, and discovery-prompt accuracy stays well above chance at a retention of $0.74$.

\begin{table}[tb]
\centering
\small
\setlength{\tabcolsep}{4pt}
\begin{tabular}{@{}p{0.34\columnwidth}p{0.57\columnwidth}@{}}
\toprule
Item & Setting \\
\midrule
Model & Qwen3-8B, bfloat16, eager attention, frozen weights \\
Components ($n=1{,}188$) & $1{,}152$ query heads and $36$ MLP blocks across $36$ layers \\
Verifier & raw faithfulness verifier $s_0$ at threshold $\alpha=0.8$ \\
Actor architecture & $2377\!\to\!512\!\to\!512\!\to\!2377$, $\tanh$ activations \\
Episode structure & add, remove, or stop; parent-support start; horizon $1{,}188$ \\
Group return & coverage measure $\mathcal R_\mu$ ($p=0.7$), l2o baseline, mean-absolute scaling \\
Training per pass & $K=10$, $3$ groups per update, $8$ updates \\
Optimization & Adam, learning rate $10^{-3}$, entropy weight $0.02$, max grad norm $1.0$ \\
Support extraction & union of canonicalized circuits across $24$ proposals per task \\
Batching & $512$ prompt-mask sequences per forward pass \\
Seed and hardware & seed $0$; one NVIDIA L40 GPU \\
\bottomrule
\end{tabular}
\caption{Configuration for the three-level hierarchical circuit search on Qwen3-8B.}
\label{tab:hier-config}
\end{table}

\subsection{Hierarchical circuit recovery on Qwen3-8B}
\label{app:hierarchical}
\begin{table}[tb]
\centering
\footnotesize
\setlength{\tabcolsep}{2pt}
\begin{tabular}{@{}lll@{}}
\toprule
Family & Subtask & Leaves \\
\midrule
Knowledge & Geography & country capital; country continent \\
          & Culture & country language; country currency \\
          & Chemistry & element symbol; element state \\
Arithmetic & Basic & addition; subtraction \\
           & Comparison & larger number; smaller number \\
           & Sequence & next number; previous number \\
Grammar & Agreement & subject-verb agreement; article number \\
        & Tense & past tense; third person \\
        & Number & plural noun; singular noun \\
Temporal & Days & day after; day before \\
         & Months & month after; month before \\
         & Alphabet & letter after; letter before \\
\bottomrule
\end{tabular}
\caption{Three-level behavior taxonomy ($4$ families, $12$ subtasks, and $24$ leaf behaviors) used for hierarchical circuit discovery on Qwen3-8B.}
\label{tab:hier-taxonomy-app}
\end{table}

To scale circuit discovery to Qwen3-8B ($36$ layers, $32$ query heads and $1$ MLP per layer, totaling $n=1{,}188$ components), we organize $24$ behaviors into a three-level hierarchy of $4$ top-level families, $12$ intermediate subtasks, and $24$ leaf tasks (Table~\ref{tab:hier-taxonomy-app}, where parent tasks pool the probe prompts of their children). Rather than restricting each child search to a single greedy parent circuit, which would prune away alternative routing pathways early, MWRL allows us to define each parent task's \emph{support} as the union of all 1-minimal circuits recovered across $24$ proposals for that parent task. Pass 1 searches over all $1{,}188$ components for the $4$ families; Pass 2 warm-starts from the family policy and restricts each subtask search to its parent family's support; and Pass 3 warm-starts from the subtask policy and restricts each leaf search to its parent subtask's support (Table~\ref{tab:hier-config}; since $K=10$, deletion credits are computed via exact inclusion--exclusion). We recover component supports for knowledge, arithmetic, and temporal tasks, but none for grammar (Figure~\ref{fig:hiertree}).

\paragraph{Taxonomy and probes.}
A subtask uses the pooled probes from its two leaves. A family uses the pooled probes from its three subtasks. Most leaves contain ten fixed probes. Element state contains eight, and each day relation contains six. The next prompt within each leaf supplies the cyclic counterfactual. Evaluation compares the first token of the correct answer with the first token of this counterfactual answer.

Figure~\ref{fig:hiertree} visualizes the recovered hierarchy and three representative leaf supports ($36$ to $81$ components out of $1{,}188$), while Figure~\ref{fig:hierconfusion} evaluates cross-task transfer by activating each nonempty leaf support in isolation on every leaf probe set for which the dense model exceeds $0.5$ correct-versus-foil accuracy. The supports retain much more accuracy on their own leaves than outside their task family: accuracy on a support's own leaf averages $0.86$ and accuracy outside its family averages $0.08$. Arithmetic supports transfer across most of their leaves; the recovered supports cover only part of the knowledge and temporal tasks. For a leaf without a recovered support, this finite search did not find a circuit, and a sparse circuit for that leaf may still exist.

\FloatBarrier

\section{Related work}
\label{app:related}

\subsection{Symbolic and oracle minimal-set discovery}
Enumerating minimal sufficient subsets is a classic problem across Boolean logic, model-based diagnosis, and combinatorics, appearing as prime implicant enumeration~\citep{quine1952,mccluskey1956,crama2011}, minimal diagnosis and Minimal Unsatisfiable Subset (MUS) extraction~\citep{reiter1987,marquessilva2013msmp,liffiton2016}, and monotone hypergraph transversal (dualization) algorithms~\citep{berge1989,eiter1995,eiter2003,fredman1996,murakami2014}. While white-box methods exploit explicit clause or hypergraph structure, interactive oracle algorithms extract minimal witnesses by querying a deterministic monotone predicate sequentially~\citep{bioch1995,junker2004,janota2016}, and recent RL approaches train graph neural policies to reduce solver calls during MUS enumeration~\citep{ijima2026hymuse}. These methods provide exact ground truth on enumerable benchmarks and inspire our problem formulation, whereas MWRL addresses how to train a parameterized policy from batch verifier bits and amortize minimal-witness discovery across contexts.

\subsection{Sufficient explanations and sparse circuits}
In explainable AI (XAI), sufficiency-based explanations seek minimal input feature subsets that preserve a model's prediction, including heuristic local search in Anchors~\citep{ribeiro2018} and Sufficient Input Subsets~\citep{carter2019}, as well as formal abductive explanations (prime implicants of classifier decisions)~\citep{shih2018,ignatiev2019}. Mechanistic interpretability poses the same question over internal network components, using activation patching~\citep{wang2023}, greedy faithfulness pruning in ACDC~\citep{conmy2023}, or fast gradient approximations in EAP~\citep{syed2024} to isolate functional subnetworks (distinct from model-compression pruning methods such as HardConcrete~\citep{louizos2018} and Wanda~\citep{sun2024}). As highlighted by~\citet{meloux2025}, neural circuit explanations are rarely unique: multiple incomparable subnetworks can implement the same behavior. MWRL turns this non-uniqueness from an ambiguity into the explicit learning target, recovering the family of minimal circuits rather than a single greedy solution.

\subsection{Diverse and relational reinforcement learning}
Standard policy-gradient methods such as REINFORCE~\citep{williams1992,sutton1999} and PPO~\citep{schulman2017} optimize expected return from sampled trajectories, and group-normalized RLVR estimators including GRPO~\citep{shao2024}, RLOO~\citep{ahmadian2024}, and MaxRL~\citep{tajwar2026maxrl} couple sampled rollouts through group normalization or baselines. Their credit does not depend on whether one terminal set contains another or repeats the same witness. Pass@$k$ policy optimization values a group of samples through its best member~\citep{walder2025pkpo}; under a binary reward it depends on the policy only through its success probability and falls under Theorem~\ref{thm:collapse}. A uniformity penalty over correct answers counters the indifference of group-relative objectives among them~\citep{lochab2026ucpo}; it targets a law proportional to a local reward, the verifier bit itself, and also falls under Theorem~\ref{thm:collapse}. When correctness is closed under supersets, a uniform law over correct answers gives each redundant superset the weight of a minimal answer. Similarly, intrinsic exploration bonuses~\citep{pathak2017}, Quality-Diversity (QD) algorithms~\citep{lehman2011,mouret2015}, MaxEnt RL~\citep{haarnoja2018}, and GFlowNets~\citep{bengio2021,malkin2022} promote diversity via action entropy, external behavior descriptors, or local reward matching. Exploration bonuses keep the objective proposal-separable asymptotically, so Theorem~\ref{thm:collapse} continues to apply to them; entropy and quality-diversity measure difference with a quantity supplied from outside the verifier, and a diversity measure that ignores containment counts redundant supersets as distinct discoveries. Entropy over the certified region is in fact maximized away from the antichain, since the up-set of a small witness is outweighed by the many larger sets above it (Figure~\ref{fig:failure}a). All our baselines use an entropy bonus (Table~\ref{tab:maxsat-config}), and the GFlowNet rows of Table~\ref{tab:maxsat} are the strongest member of this family that uses only the verifier bit. (Order-preserving GFlowNets~\citep{chen2024opgfn} incorporate pairwise order comparisons to concentrate on top-ranked candidates, whereas MWRL uses the subset order to recover the entire antichain of minimal elements.)

\subsection{Submodular RL}
Submodular RL~\citep{prajapat2024} and Global RL~\citep{desanti2024} generalize reinforcement learning to optimize non-additive, submodular utilities over state or trajectory visitation distributions supplied by the environment. MWRL shares this non-additive mathematical structure: our group coverage valuation $\mathcal R_\mu(U_K) = \mu(\bigcup_{i:y_i=1} \upset{S_i})$ is monotone and submodular in the set of accepted proposals, and our deletion credit $A_i$ is its exact discrete marginal gain. Crucially, however, MWRL requires no external submodular reward from the environment; it constructs this submodular structure internally from binary pass/fail verifier bits and the subset lattice order.

\subsection{Hypervolume and set-level objectives}
In multi-objective optimization, the hypervolume indicator measures the volume of objective space dominated by a Pareto approximation set, and each solution's exclusive hypervolume contribution measures the volume lost when that solution is removed~\citep{beume2007,vanmoffaert2013,guerreiro2021,zhang2023hv}. On the Boolean lattice $2^{\mathcal E}$, our geometric coverage measure $\mathcal R_\mu(U_K)$ admits an exact interpretation as a discrete hypervolume: if we map each accepted subset $S \subseteq \mathcal E$ to its binary indicator vector in $\{0,1\}^{\mathcal E}$ (minimizing each coordinate toward $0$) against the reference point $((1-p)^{-1}, \dots, (1-p)^{-1})$, then $S$ dominates all of its supersets, and the minimal-witness antichain $\Msf(c)$ is the exact Pareto front of $F(c)$. Splitting every coordinate at $1$ cuts the box $[0,(1-p)^{-1}]^{\mathcal E}$ into cells indexed by the subsets $C\subseteq\mathcal E$ of coordinates above $1$. The region dominated by $S$ covers the cell $C$ exactly when $S\subseteq C$, and that cell has volume $(p/(1-p))^{|C|}=(1-p)^{-|\mathcal E|}\mu(\{C\})$. The hypervolume of a group's successes is therefore $(1-p)^{-|\mathcal E|}\mathcal R_\mu(U_K)$, and each exclusive contribution is the deletion credit up to the same constant. Theorem~\ref{thm:rep} shows that any group value unchanged by deleting duplicate or redundant successes can be written as a function of the certified region; under the geometric measure its valuation coincides with a hypervolume. Whereas multi-objective hypervolume methods assume continuous objective vectors are observed directly, MWRL operates from a single binary verifier bit per rollout to uncover the discrete Pareto front $\Msf(c)$.

\section{Scope of application}
\label{app:scope}

Minimal-witness identification applies across a wide spectrum of scientific and computational problems where an outcome is tested by a binary verifier over subsets of candidate factors. Below we summarize the structural conditions for applying MWRL, catalog domain instantiations in Table~\ref{tab:scope}, and clarify which problems fall outside our formulation.

\begin{table*}[t]
\centering
\footnotesize
\setlength{\tabcolsep}{6pt}
\begin{tabular}{@{}p{0.185\textwidth}p{0.205\textwidth}p{0.31\textwidth}p{0.20\textwidth}@{}}
\toprule
Setting & Ground set $\mathcal{E}$ & Verifier bit $s(S)=1$ & Base measure $\mu$ \\
\midrule
\multicolumn{4}{@{}l}{\emph{Instantiated in this paper}}\\
MaxSAT witnesses & Boolean variables & setting variables in $S$ to true satisfies all clauses &
geometric ($p=0.7$) \\
Reaction conditions & condition dimensions varied from baseline & a tabulated
assignment within $S$ reaches the target yield & geometric ($p=0.7$) \\
Sparse circuits & attention heads and MLP blocks & patching all components outside $S$ retains $\ge \alpha$ faithfulness & geometric ($p=0.7$) \\
\midrule
\multicolumn{4}{@{}l}{\emph{Further settings admitted by the formulation}}\\
Coherent reliability & system components & simultaneous component failure causes system outage & component failure probabilities \\
Attack graphs & vulnerabilities and credentials & attacker reaches target privilege node &
exploit acquisition cost \\
Least privilege & access permissions & workload passes integration test suite & permission exposure weights \\
Why-provenance & database tuples & query over subset $S$ returns target answer tuple & tuple reliability weights \\
Candidate keys & relation attributes & functional dependency uniquely determines all tuples & attribute acquisition costs \\
Configuration faults & compiler flags or input lines & build or test execution reproduces crash & flag prevalence in corpus \\
Evidence minimization & retrieved passages or prompt segments & LLM answer remains unchanged & passage retrieval probabilities \\
Metabolic design & genes or media nutrients & cellular flux or growth exceeds viability threshold &
perturbation feasibility \\
Combination therapy & drug compounds and doses & phenotypic assay exceeds efficacy threshold &
safety or cost weights \\
Biomarker panels & diagnostic assays & classifier accuracy exceeds clinical threshold & per-assay clinical cost \\
Data attribution & training dataset subsets & target model behavior persists after retraining & corpus sampling distribution \\
\bottomrule
\end{tabular}
\caption{Representative instantiations of minimal-witness identification across computation, security, databases, and experimental science. Each row names the ground set over which proposals are formed, the event that the verifier bit indicates, and the information a domain supplies for choosing the base measure of Appendix~\ref{app:mc}.}
\label{tab:scope}
\end{table*}

\subsection{Conditions for applying MWRL}

A setting instantiates the problem of Section~\ref{sec:problem} when the following five conditions hold:
\begin{itemize}
\item \textbf{Subset structure:} Candidate proposals are subsets $S \subseteq \mathcal E$ of a finite ground set of controllable factors $\mathcal E$.
\item \textbf{Binary verifier interface:} The environment evaluates each proposed subset via a pass/fail bit $s_c(S) \in \{0,1\}$. Additional verifier calls permit the counting measure of Appendix~B, and the coverage measure needs the bit alone.
\item \textbf{Exact or approximate monotonicity:} Adding elements to a sufficient subset preserves sufficiency (or, when minor interference occurs, 1-minimal sets remain meaningful targets; see Appendix~\ref{app:nonmono}). Under a verifier that violates monotonicity, a canonicalizer returns 1-minimal sets, which need not be minimal.
\item \textbf{Plurality of mechanisms:} The ground-truth antichain $\Msf(c)$ contains multiple incomparable minimal witnesses that we wish to recover simultaneously. (When $|\Msf(c)|=1$, a simple cardinality penalty already suffices.)
\item \textbf{Costly verification over related contexts:} Calls to $s$ are expensive against the size of $2^{\mathcal{E}}$, and contexts arrive from a distribution over which a policy can amortize. For predicates exposed as formulas, the symbolic enumerators of Appendix~\ref{app:related} recover $\Msf$ without training.
\end{itemize}
The base measure $\mu$ encodes domain knowledge by assigning weights to sets in the lattice, and these weights determine which incomparable witnesses an optimal finite-budget planner selects (Appendix~\ref{app:vi}). Many domains naturally supply an informative prior for $\mu$: in reliability engineering, $\mu$ is the physical product distribution of component failures; in retrieval-augmented generation, $\mu$ is the retriever's inclusion law; and in security or clinical panels, $\mu$ can weight elements by their acquisition cost or safety risk.

\subsection{Settings where the antichain is already the classical target}
In several classical disciplines, the minimal-witness antichain $\Msf$ and the up-set coverage measure $\mathcal R_\mu(U)$ are the exact quantities practitioners already seek:
\begin{itemize}
\item \textbf{Reliability and attack graphs:} In coherent reliability systems~\citep{barlow1975}, minimal cut sets form an antichain whose up-set union under the component failure measure $\mu$ equals the exact system failure probability, and the contribution of one cut set, what that union loses when the cut set is removed, is our deletion credit $A_i$; these are the coverage measure of Section~\ref{sec:method} and its deletion credit, computed by the identities of Appendix~\ref{app:mc}. In cybersecurity attack graphs~\citep{sheyner2002}, minimal exploit chains form an antichain where patching one minimal chain leaves all incomparable chains vulnerable, making full antichain recovery mandatory for defense.
\item \textbf{Database provenance and epidemiology:} In database theory, \emph{why-provenance} defines the minimal witness basis of a query result as the antichain of minimal tuple subsets that produce it~\citep{buneman2001}, and candidate keys form the antichain of minimal attribute sets that functionally determine a relation. Similarly, Rothman's sufficient-component cause model in epidemiology defines each disease pathway as a minimal sufficient set of component causes~\citep{rothman1976}, paralleling minimal hitting sets in model-based diagnosis~\citep{reiter1987}.
\end{itemize}
These fields supply ground truth on enumerable instances and independent evidence that the coverage measure is the quantity their practitioners compute. Their predicates are symbolic and are already served by enumeration.

\subsection{Settings that require a learned policy}
When the verifier requires executing a complex simulation, compilation, neural forward pass, or wet-lab assay, symbolic dualization is inapplicable and MWRL enables amortized discovery:
\begin{itemize}
\item \textbf{Configuration debugging and least-privilege security:} Isolating irreducible combinations of compiler flags or input lines that trigger a software crash (generalizing single-path delta debugging~\citep{zeller2002} to multiple distinct fault modes), or identifying alternative minimal permission sets that allow a cloud workload to pass integration tests while minimizing attack surface.
\item \textbf{Evidence minimization in RAG and long-context LLMs:} Identifying minimal subsets of retrieved documents or prompt exemplars sufficient to sustain an LLM's answer (where minor context distraction places the verifier in the approximately monotone regime of~\citet{liu2024lost} and Appendix~\ref{app:nonmono}). Recovering multiple incomparable minimal evidence sets quantifies the evidential redundancy backing a model's claim.
\item \textbf{Perturbation biology and combination therapy:} Discovering minimal gene sets, nutrient media, or transcription factor cocktails sufficient for cellular viability, computable in silico by flux balance analysis~\citep{orth2010}, or for reprogramming (such as the classic 1-minimality screen that isolated the four Yamanaka factors from a pool of $24$ candidates~\citep{takahashi2006}), as well as identifying alternative minimal drug combinations that reach an efficacy threshold so clinicians can select the least toxic regimen.
\item \textbf{Training data attribution:} Identifying minimal subsets of training examples sufficient to induce a specific model behavior~\citep{koh2017}, where amortizing a conditional policy across target behaviors avoids prohibitive per-behavior retraining sweeps.
\end{itemize}

\subsection{Settings outside the formulation}
Finally, three classes of problems fall outside our formulation: (1) budgeted submodular maximization tasks (such as sensor placement or influence maximization~\citep{prajapat2024}) that seek a single cardinality-constrained subset maximizing a continuous utility; (2) minimum-cardinality problems (such as finding the single smallest unsatisfiable core) that deliberately discard larger minimal witnesses; and (3) continuous optimization settings where the verifier returns graded real-valued rewards that carry fine-grained trade-off information beyond a binary sufficiency threshold.

\section{Ethical statement}
\label{app:ethics}

Minimal-witness identification is a general algorithmic framework for isolating irreducible sufficient subsets from binary verifier feedback. In experimental sciences such as chemistry and biology (Appendix~F), replacing a benchmark lookup table with an automated laboratory assay allows an MWRL policy to propose minimal condition or perturbation sets directly. Because MWRL interacts with the domain solely through a binary pass/fail bit, all domain-specific safety constraints, reagent restrictions, and ethical oversight must be enforced within the design of the verifier and the definition of the ground set $\mathcal E$. MWRL can accelerate search within a sanctioned experimental space, but it cannot invent new actions outside $\mathcal E$ or observe unthresholded assay quantities.

Furthermore, any recovered witness family is guaranteed to be sufficient only with respect to the specific verifier and probe distribution used during discovery. For circuits, we test whether acceptance on the discovery probes extends to held-out questions from the same subject; the recovered circuits fail to pass the verifier again on those questions, so sufficiency depends on the verifier and the probe distribution (Appendix~G.1). Consequently, candidate mechanisms discovered in chemical or biological applications must be independently validated on target populations before being interpreted as causal mechanisms. Conversely, recovering the complete antichain of minimal witnesses provides an important safeguard against premature scientific closure~\citep{meloux2025}: rather than reporting a single smallest witness as though it were the unique explanation, MWRL exposes alternative pathways and backup mechanisms, allowing practitioners to choose safer or less costly interventions.

We used AI tools mainly for language polishing, grammar correction, and novelty checks
against prior work.